\documentclass[journal]{IEEEtran}

\usepackage{xcolor,soul,framed} %,caption
\usepackage{booktabs}
\usepackage{multirow}
 \usepackage{cite}
 \usepackage{color}

\colorlet{shadecolor}{yellow}
\usepackage[pdftex]{graphicx}
\graphicspath{{../pdf/}{../jpeg/}}
\DeclareGraphicsExtensions{.pdf,.jpeg,.png}

\usepackage[cmex10]{amsmath}
\usepackage{amsfonts,amssymb}
\usepackage{array}
\usepackage{mdwmath}
\usepackage{mdwtab}
\usepackage{eqparbox}
\usepackage{url}
\usepackage{amsthm}
\newtheorem{definition}{Definition}
\newtheorem{theorem}{Theorem}
\usepackage{algorithmic}
\usepackage{algorithm}
\usepackage{graphicx}
\usepackage{subfigure}
\begin{document}
\bstctlcite{IEEEexample:BSTcontrol}
    \title{Fast Greedy Subset Selection from Large Candidate Solution Sets in Evolutionary Multi-objective Optimization
}
  \author{Weiyu~Chen,~\IEEEmembership{}
      Hisao~Ishibuchi,~\IEEEmembership{Fellow,~IEEE,}
      and Ke~Shang~\IEEEmembership{}% <-this % stops a space

  \thanks{This paper is an expanded paper from the IEEE Congress on Evolutionary Computation held on July 19-24, 2020 in Glasgow, UK. This work was supported by National Natural Science Foundation of China (Grant No. 61876075), Guangdong Provincial Key Laboratory (Grant No. 2020B121201001), the Program for Guangdong Introducing Innovative and Enterpreneurial Teams (Grant No. 2017ZT07X386), Shenzhen Science and Technology Program (Grant No. KQTD2016112514355531), the Program for University Key Laboratory of Guangdong Province (Grant No. 2017KSYS008).}
  \thanks{Corresponding author: Hisao Ishibuchi}
  \thanks{Weiyu Chen, Hisao Ishibuchi and Ke Shang are with the Department of Computer Science and Engineering,  Southern University of Science and Technology, Shenzhen, China (e-mail: 11711904@mail.sustech.edu.cn; hisao@sustech.edu.cn; kshang@foxmail.com).}% <-this % stops a space
}

\markboth{IEEE TRANSACTIONS ON EVOLUTIONARY COMPUTATION, VOL.~XX, NO.~XX, AUGUST~2020
}{Chen \MakeLowercase{\textit{et al.}}: Fast Greedy Subset Selection}

% ====================================================================
\maketitle

% === ABSTRACT ====================================================================
% =================================================================================
\begin{abstract}
%\boldmath
Subset selection is an interesting and important topic in the field of evolutionary multi-objective optimization (EMO). Especially, in an EMO algorithm with an unbounded external archive, subset selection is an essential post-processing procedure to select a pre-specified number of solutions as the final result. In this paper, we discuss the efficiency of greedy subset selection for the hypervolume, IGD and IGD+ indicators. Greedy algorithms usually efficiently handle subset selection. However, when a large number of solutions are given (e.g., subset selection from tens of thousands of solutions in an unbounded external archive), they often become time-consuming. Our idea is to use the submodular property, which is known for the hypervolume indicator, to improve their efficiency. First, we prove that the IGD and IGD+ indicators are also submodular. Next, based on the submodular property, we propose an efficient greedy inclusion algorithm for each indicator. Then, we demonstrate through computational experiments that the proposed algorithms are much faster than the standard greedy subset selection algorithms.
\end{abstract}

% === KEYWORDS ====================================================================
% =================================================================================

% \fi

% ====================================================================
% ====================================================================
% ====================================================================

% === I. INTRODUCTION =============================================================
% =================================================================================
\section{Introduction}

\IEEEPARstart{E}{volutionary} multi-objective optimization (EMO) algorithms aim to optimize $m$ potentially conflicting objectives concurrently. To compare solutions, we usually use the Pareto-dominance relation \cite{1multicriteria}. In a multi-objective minimization problem with $m$ objectives $f_i(x), i = 1, 2, ..., m$, solution $a$ dominates solution $b$ (i.e., $a \prec b$) if and only if $\forall i \in \{1,2,…,m\}, f_i(a) \leq f_i (b)$ and $\exists j \in \{1,2,…,m\},f_j(a) < f_j(b)$. Its weaker version is defined by only the first condition: $a$ weakly dominates $b$ (i.e., $a \preceq b$) if and only if $\forall i\in \{1,2,…,m\}, f_i(a) \leq f_i(b)$. This relation includes the case where $a$ and $b$ are exactly the same in the $m$-dimensional objective space (i.e., $f_i(a) = f_i(b)$ for $i = 1, 2, ..., m$). 

A solution that is not dominated by any other feasible solutions of a multi-objective problem is called a Pareto optimal solution. A Pareto front of the multi-objective problem is the set that contains the corresponding objective values of all Pareto optimal solutions. If no solution in a solution set is dominated by any other solutions, the solution set is referred to as a non-dominated solution set. 

The Pareto dominance relation has been extended to the dominance relation between two sets \cite{2Zitzler}: A solution set $A$ dominates solution set $B$ if and only if $\forall b \in B, \exists a \in A$, such that $a \prec b$. A better relation \cite{2Zitzler} is defined using the weak Pareto dominance relation between two non-dominated solution sets $A$ and $B$ as follows: $A$ is better than $B$ if and only if $\forall b \in B, \exists a \in A, a \preceq b$ and $A \neq B$.

Since the number of Pareto optimal solutions can be very large for combinatorial optimization problems and infinity for continuous optimization problems, subset selection is an essential topic in the EMO field. It is involved in many phases of EMO algorithms. (i) In each generation, we need to select a pre-specified number of solutions from the current and offspring populations for the next generation. (ii) After the execution of EMO algorithms, the final population is usually presented to the decision-maker. However, if the decision-maker does not want to examine all solutions in the final population, we need to choose only a small number of representative solutions for the decision-makers. (iii) Since many good solutions are discarded during the execution of EMO algorithms \cite{3Hisao}, we can use an unbounded external archive (UEA) to store all non-dominated solutions examined during the execution of EMO algorithms. In this case, a pre-specified number of solutions are selected from a large number of non-dominated solutions in the archive.

Formally, a subset selection problem can be defined as follows:
\begin{definition}[Subset Selection Problem]
Given an $n$-point set $V \subset  \mathbb{R}^m$, a performance indicator $g: 2^{|V|} \rightarrow \mathbb{R}$ and a positive integer $k$ ($k < n$), maximize $g(S)$ subject to $S \subset V$  and $|S| \leq k$.
\end{definition}

Many subset selection algorithms have been proposed. They can be classified into the following three categories: (i) exhaustive search algorithms, (ii) evolutionary algorithms, and (iii) greedy algorithms. Exhaustive search algorithms search for the optimal solution subset. The performance indicator value of the optimal solution subset is always called OPT. Since the subset selection problem is NP-hard \cite{4Natarajan,5Davis}, it is impractical to find the optimal solution subset unless the candidate solution set is small. In practice, evolutionary algorithms and greedy algorithms are usually used. 

Greedy algorithms can be further divided into greedy inclusion algorithms and greedy removal algorithms. The greedy inclusion algorithms select solutions from \(V\) one by one. In each iteration, the solution that leads to the largest improvement of the performance indicator \(g\) is selected until the required number of solutions are selected. If the performance indicator is submodular, the obtained solution set can achieve a $(1-1/e)$-approximation ($e$ is the natural constant) to the optimal subset \cite{6Nemhauser}. This means that the ratio of the performance indicator value of the obtained solution set to the performance indicator value of the optimal solution set is not less than $(1-1/e)$. In contrast to greedy inclusion algorithms, greedy removal algorithms discard one solution with the least improvement of the performance indicator \(g\) in each iteration. Although greedy removal algorithms are widely used and have shown good performance \cite{7Bradstreet}, currently there is no theoretical guarantee for their approximation accuracy.

When the required set size \(k\) is close to the size of $V$ (i.e., when the number of solutions to be removed is small), greedy removal algorithms are faster than greedy inclusion algorithms. However, when $k$ is relatively small in comparison with the size of $V$, greedy removal algorithms are not efficient since they need to remove a large number of solutions.

Evolutionary algorithms have also been applied to subset selection problems \cite{8Qian,9Hisao}. An arbitrary subset of $V$ can be encoded as a binary string $b$ of length $|V|$. In the string, a one “1” indicates the inclusion of the corresponding solution in the subset and a zero “0” indicates the exclusion of the solution. A population of subsets (i.e., binary strings) is improved by an evolutionary algorithm. If the performance indicator \(g\) is submodular, some algorithms such as POSS \cite{8Qian} have also been proved to be $(1-1/e)$-approximation algorithms.

Subset selection can use different selection criteria (i.e., different solution set evaluation criteria). From this point of view, subset selection algorithms can be categorized as distance-based subset selection \cite{10Tanabe,11Singh,12Chen}, clustering-based subset selection \cite{11Singh}, hypervolume subset selection \cite{13hype,14Bringmann,15Kuhn,16Guerreiro}, IGD subset selection \cite{17Ishibuchi}, IGD+ subset selection \cite{17Ishibuchi}, and $\epsilon$-indicator subset selection \cite{14Bringmann,18bringmann}. Some of these indicators require a large computation load. For example, the calculation of the hypervolume indicator is \#P-hard with respect to the number of objectives \cite{19bringmann2010approximating}. 

Selecting the best solution from all the evaluated solutions as the final solution is a standard choice in single-objective optimization. The best solution is usually stored and updated during the execution of a single-objective algorithm. However, in multi-objective optimization, it is not easy to store the best solution set since the quality of each solution is relatively measured based on its relation with other solutions. It is possible that the same solution is evaluated as the best in one generation and the worst in another generation. Experimental results in \cite{Miqing} show that some good solutions can be discarded during the evolution process (thus the final solution set is not the best solution set). Therefore, recently, subset selection from all the evaluated solutions has been used in some studies to find better solution sets than the final population \cite{3Hisao, 10Tanabe, 17Ishibuchi,12Chen,landscape,pang2020algorithm,reverse,TANABE2018}. The selected solution sets usually have better indicator values than the final population.

Although subset selection from all the evaluated solutions is effective, it is time-consuming. When an EMO algorithm is applied to a many-objective problem, a huge number of non-dominated solutions are usually included in the evaluated solutions. That is, the size of the candidate solution set for subset selection is huge. For example, the size of  the UEA obtained by NSGA-II on the nine-objective car  cab  design  problem \cite{NSGAIII}  after 200 generations is around 20,000. In this case, subset selection needs long computation time even when greedy algorithms are used. The focus of this paper is to decrease the computation time of greedy subset selection algorithms, which improves the usefulness and applicability of EMO algorithms in real-world scenarios by efficiently selecting better solution sets than the final population.

In this paper, we propose new greedy inclusion algorithms for the hypervolume, IGD and IGD+ subset selection. The submodularity \cite{6Nemhauser} of the hypervolume, IGD and IGD+ indicators is the key to make the proposed algorithms efficient and applicable to large candidate solution sets with many objectives. Based on the submodularity of these indicators, we can reduce the unnecessary contribution calculations of solutions without changing the results (i.e., without changing the obtained solution sets) of the corresponding greedy subset selection algorithms. Experimental results show that the proposed idea drastically improves the efficiency of greedy subset selection from large candidate solution sets of many-objective problems.

The followings are the main contributions of this paper.
\begin{itemize}
\item We explain the submodular property of set functions, and prove that the IGD and IGD+ indicators have the submodular property. Based on our proof, the theoretical approximation ratio of the greedy algorithm based on each indicator can be obtained. 
\item By exploiting the submodular property of the three indicators (i.e., hypervolume, IGD, and IGD+), we propose an efficient greedy inclusion algorithm for subset selection based on each indicator.
\item We demonstrate that the computation time for subset selection can be significantly decreased by the proposed algorithms through computational experiments on frequently-used test problems and some real-world problems.
\end{itemize}

The remainder of this paper is arranged as follows. In Section II, we explain hypervolume-based subset selection. In Section III, we discuss the use of the IGD and IGD+ indicators for subset selection. In Section IV, we propose an efficient greedy inclusion algorithm for each indicator. In Section V, the proposed algorithms are compared with the standard greedy algorithms. Finally, we conclude this paper in Section V. 

This paper is an extended version of our conference paper \cite{45HSS}. In our conference paper, we proposed an efficient greedy inclusion algorithm using the submodular property of the hypervolume indicator. In this paper, we propose efficient algorithms for the IGD and IGD+ indicators after proving their submodular property. More experimental results are also reported in this paper in order to demonstrate the efficiency of the proposed submodular property-based algorithms for the three performance indicators.

% === II. Harmonically-Terminated Power Rectifier Analysis ========================
% =================================================================================
\section{Hypervolume Subset Selection}

\subsection{Hypervolume indicator and hypervolume contribution}
The hypervolume indicator \cite{20knowles2003bounded, 21Zitzler, KeHypervolume} is a widely used metric to evaluate the diversity and convergence of a solution set. It is defined as the size of the objective space which is covered by a set of non-dominated solutions and bounded by a reference set \(R\). Formally, the hypervolume of a solution set \(S\) is defined as follows:
\begin{equation}
HV(S) := \int_{z \in \mathbb{R}^m} A_s(z) dz,
\end{equation}
where \(m\) is the number of dimensions and $A_s: \mathbb{R}^m \rightarrow \{0, 1\}$ is the attainment function of $S$ with respect to the reference set $R$ and can be written as
\begin{equation}
A_s(z) =\left\{
\begin{aligned}
1 &          &if\ \exists \  s \in S, r \in R : f(s) \preceq z \preceq r,\\
0 &          & otherwise.
\end{aligned}
\right.
\end{equation}

Calculating the hypervolume of a solution set is a \#P-hard problem \cite{19bringmann2010approximating}. A number of algorithms have been proposed to quickly calculate the exact hypervolume such as Hypervolume by Slicing Objectives (HSO) \cite{22durillo2010jmetal,23durillo2011jmetal}, Hypervolume by Overmars and Yap (HOY) \cite{24,25overmars1991new,26beume2009s}, and Walking Fish Group (WFG) \cite{27while2011fast}. Among those algorithms, WFG has been generally accepted as the fastest one. 
The hypervolume contribution is defined based on the hypervolume indicator. The hypervolume contribution of a solution \(p\) to a set \(S\) is
\begin{equation}
HVC(p, S) = HV(S \cup \{p\}) - HV(S).
\end{equation}

Fig. 1 illustrates the hypervolume of a solution set and the hypervolume contribution of a solution to the solution set in two dimensions. The grey region is the hypervolume of the solution set \( S = \{a,b,c,d,e\} \) and the yellow region is the hypervolume contribution of a solution \(p\) to \(S\).

\begin{figure}[htbp]
\centering
\includegraphics[width= 0.35\textwidth, trim=10 20 0 0,clip]{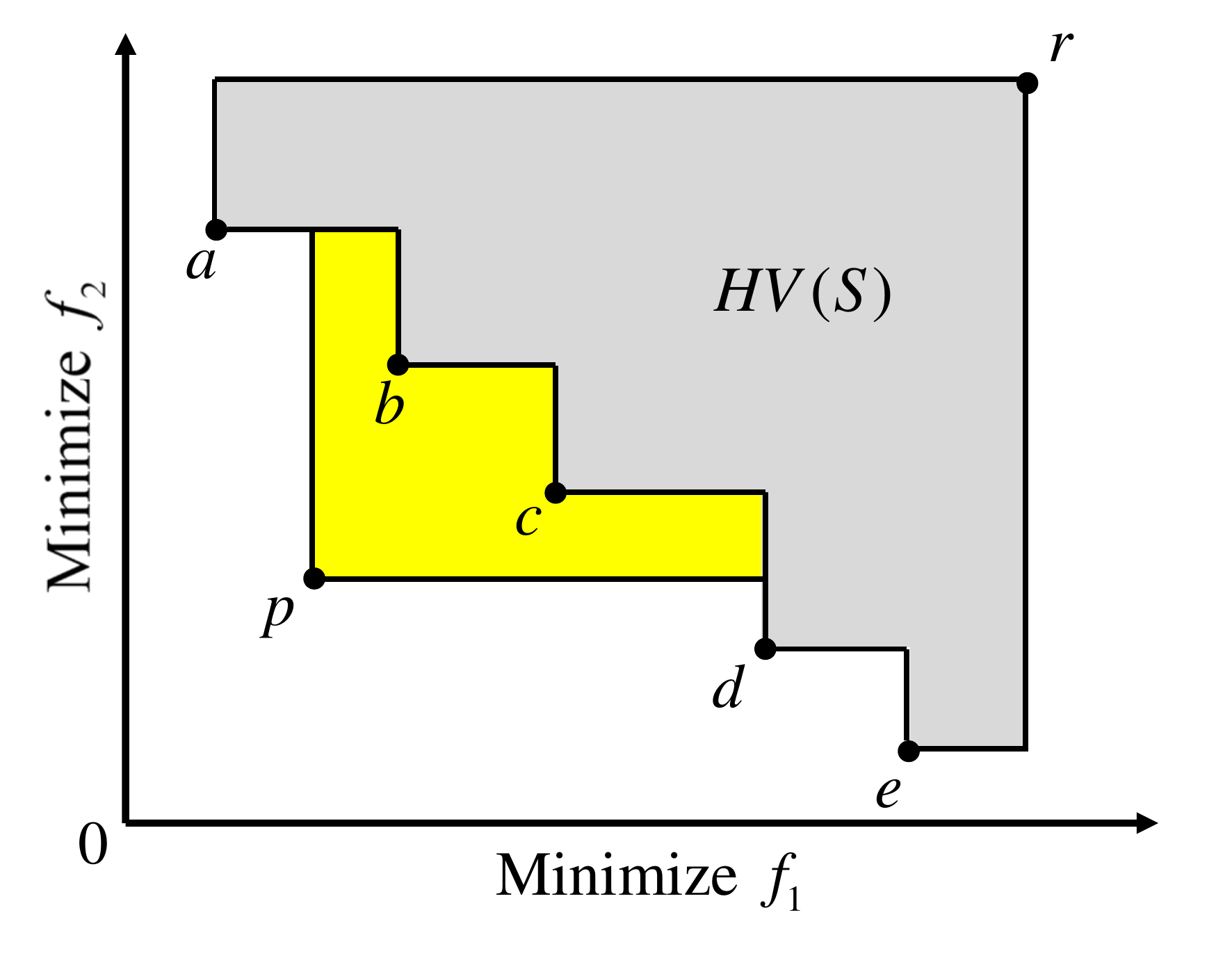}
\caption{The hypervolume of the solution set $S = \{a,b,c,d,e\}$ and the hypervolume contribution of \(p\) to the solution set \(S\) for a two-objective minimization problem.}
\end{figure}

Note that calculating the hypervolume contribution based on its definition in (3) requires hypervolume calculation twice, which is not very efficient. Bringmann and Friedrich \cite{19bringmann2010approximating} and Bradstreet et al. \cite{29Bradstreet2008fast} proposed a new calculation method to reduce the amount of calculation. The hypervolume contribution is calculated as
\begin{equation}
HVC(p, S) = HV(\{p\}) - HV(S'),
\end{equation}
where
\begin{equation}
S' = \{limit(s,p)|s \in S\},
\end{equation}
\begin{equation}
\begin{split}
&limit((s_1, ..., s_m), (p_1, ..., p_m)) \\
&= (worse(s_1, p_1),...,worse(s_m, p_m)).
\end{split}
\end{equation}

In this formulation $worse\left(s_i, p_i\right)$ takes the larger value for minimization problems. Compared to the straightforward calculation method in (3), this method is much more efficient. The hypervolume of one solution (i.e., \( HV\left(\{p\}\right)\)) can be easily calculated. We can also apply the previous mentioned HSO \cite{22durillo2010jmetal,23durillo2011jmetal}, HOY \cite{24,25overmars1991new,26beume2009s} and WFG \cite{27while2011fast} to calculate the hypervolume of a reduced solution set \(S’\) (i.e., \(HV\left(S'\right)\)).

Let us take Fig. 2 as an example. Suppose we want to calculate the hypervolume contribution of solution $p$ to a solution set $S = \{a,b,c,d,e\}$. First, for each solution in $S$, we replace each of its objective values with the corresponding value from solution $p$   if the value of $p$ is larger (i.e., we calculate $limit (a, p), ..., limit(e, p)$). This leads to $S’ = \{a’,b’,c’,d’,e’\}$. After the replacement, \(e’\) is dominated by \(d’\). Thus \(e’\) can be removed from \( S’\) since $e’$ has no contribution to the hypervolume of \(S’\). Similarly, \(a’\) and  \(b’\) can also be removed from \( S’\).  Then, we calculate the hypervolume of $S’$ (i.e., the area of the gray region in Fig. 2) and subtract it from the hypervolume of solution $p$. The remaining yellow part is the hypervolume contribution of solution $p$.
\begin{figure}[htbp]
\centering
\includegraphics[width=0.35 \textwidth, trim=10 20 0 0,clip]{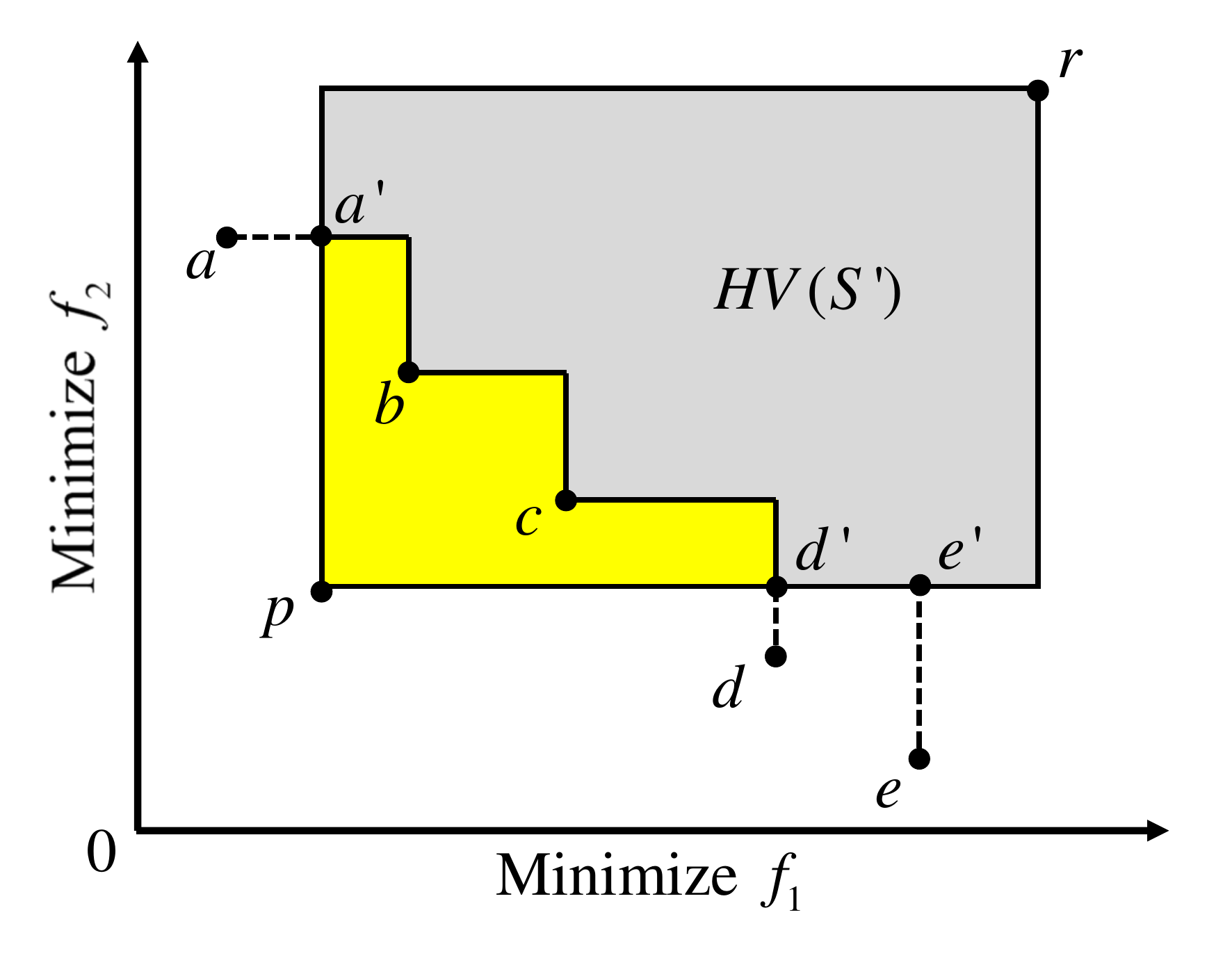}
\caption{Illustration of the efficient hypervolume contribution computation method.}
\end{figure}
\subsection{Hypervolume Subset Selection Problem }
The hypervolume subset selection problem (HSSP) \cite{13hype} uses the hypervolume indicator as the solution selection criteria (i.e., the hypervolume indicator is used as the performance indicator \(g\) in \textbf{Definition 1}). The HSSP aims to select a pre-specified number of solutions from a given candidate solution set to maximize the hypervolume of the selected solutions. 

For two-objective problems, HSSP can be solved with time complexity of \(O(nk + nlogn) \)\cite{18bringmann} and \(O((n-k)k + nlogn)\)\cite{15Kuhn}. For multi-objective problems with three or more objectives, HSSP is an NP-hard problem \cite{28Rote2016selecting}, it is impractical to try to find the exact optimal solution set when the size of the candidate set is large and/or the dimensionality of the objective space is high. In practice, some greedy heuristic algorithms and genetic algorithms are employed to obtain an approximated optimal solution set.

\subsection{Greedy Hypervolume Subset Selection}
Hypervolume greedy inclusion selects solutions from $V$ one by one. In each iteration, the solution that has the largest hypervolume contribution to the selected solution set is selected until the required number of solutions are selected. The pseudocode of greedy inclusion is shown in Algorithm 1. Since hypervolume is a submodular indicator \cite{37ulrich2012bounding}, hypervolume greedy inclusion algorithms provide a $(1-1/e)$-approximation to HSSP \cite{6Nemhauser}. 

Hypervolume greedy removal algorithms discard one solution with the least hypervolume contribution to the current solution set in each iteration. Unlike greedy inclusion, greedy removal has no approximation guarantee. It can obtain an arbitrary bad solution subset \cite{43bringmann2010efficient}. However, in practice, it usually leads to good approximations.

Although the greedy algorithm has polynomial complexity, it still takes a long computation time when the candidate set is large or the dimension is high. A lot of algorithms are proposed to accelerate the naïve greedy algorithm. For example, to accelerate the hypervolume-based greedy removal algorithm, Incremental Hypervolume by Slicing Objectives (IHSO*) \cite{29Bradstreet2008fast} and Incremental WFG (IWFG) \cite{30cox2016improving} were proposed to identify the solution with the least hypervolume contribution quickly. Some experimental results show that these methods can significantly accelerate greedy removal algorithms.

\begin{algorithm}
	\caption{Greedy Inclusion Hypervolume Subset Selection}
	\begin{algorithmic}[1]
    \REQUIRE  $V$ (A set of non-dominated solutions), $k$ (Solution subset size)
    \ENSURE $S$  (The selected subset from $V$)
     \IF{$|V| < k $}
           \STATE $S = V$
     \ELSE
  			\STATE  $S = \emptyset$
    		\WHILE{$|S| < k $}
    			\FOR {\textbf{each} $s_i$ in $V \setminus S$}
         				\STATE Calculate the hypervolume contribution of $s_i$ to $S$
                 \ENDFOR
      			\STATE $p$ = Solution in  $V\setminus S$  with the largest hypervolume contribution
      			\STATE $S = S \cup \{p\}$ 
             \ENDWHILE  
      \ENDIF
	\end{algorithmic}
\end{algorithm}

Hypervolume-based greedy inclusion/removal algorithms can be accelerated by updating hypervolume contributions instead of recalculating them in each iteration (i.e., by utilizing the calculation results in the previous iteration instead of calculating hypervolume contributions in each iteration independently). Guerreiro et al. \cite{16Guerreiro} proposed an algorithm to update the hypervolume contributions efficiently in three and four dimensions. Using their algorithm, the time complexity of hypervolume-based greedy removal in three and four dimensions can be reduced to \(O (n(n - k) + n log n) \) and \(O(n^2 (n - k)) \) respectively.

In a hypervolume-based EMO algorithm called FV-MOEA proposed by Jiang et al. \cite{31jiang2014simple}, an efficient hypervolume contribution update method applicable to any dimension was proposed. The main idea of their method is that the hypervolume contribution of a solution is only associated with a small number of its neighboring solutions rather than all solutions in the solution set. Let us suppose that one solution \( s_j \) have just been removed from the solution set \( S\), the main process of the hypervolume contribution update method in \cite{31jiang2014simple} is shown in Algorithm 2. 

\begin{algorithm}
	\caption{Hypervolume Contribution Update}
	\begin{algorithmic}[1]
    \REQUIRE  $HVC$ (The hypervolume contribution of each solution in $S$), $s_j$ (The newly removed solution)
    \ENSURE $HVC$ (The updated hypervolume contribution of each solution in \(S\))
    \FOR {\textbf{each} $s_k \in S$}
         \STATE $w = worse(s_k, s_j)$
         \STATE $W = \{limit(t, w)|t \in S\}$
         \STATE $HVC(s_k) = HVC(s_k) + HV(\{w\}) - HV(W)$
    \ENDFOR 
	\end{algorithmic}
\end{algorithm}
The \textit{worse} and \textit{limit} operations in Algorithm 2 are the same as those in Section II.A. Let us explain the basic idea of Algorithm 2 using Fig. 3. When we have a solution set \(S = \{a, b, c, d, e\}\) in Fig. 3, the hypervolume contribution of solution \( c\) is the blue area. When solution \(b\) is removed, the hypervolume contribution of \(c\) is updated as follows. The worse solution \(w\) in line 2 of Algorithm 2 has the maximum objective values of solutions \(b\) and \(c\). In line 3, firstly the limit operator changes solutions \(a\), \(d\) and \(e\) to \(a’\), \(d’\) and \(e’\). Next, the dominated solution \(e’\) is removed.  Then the solution set \(W = \{a’,d’\}\) is obtained. In line 4, the hypervolume contribution of \(c\) is updated by adding the term \(HV(\{w\})-HV(W)\) to its original value (i.e., the blue region in Fig. 3). The added term is the joint hypervolume contribution of solutions \(b\) and \( c\) (i.e., the yellow region in Fig. 3).  In this way, the hypervolume contribution of each solution is updated.

Since the \textit{limit} process reduces the number of non-dominated solutions, this updated method greatly improves the speed of hypervolume-based greedy removal algorithms. Algorithm 2 in \cite{31jiang2014simple} is the fastest known algorithm to update the hypervolume contribution in any dimension.

\begin{figure}[htbp]
\centering
\includegraphics[width=0.35 \textwidth, trim=10 13 0 0,clip]{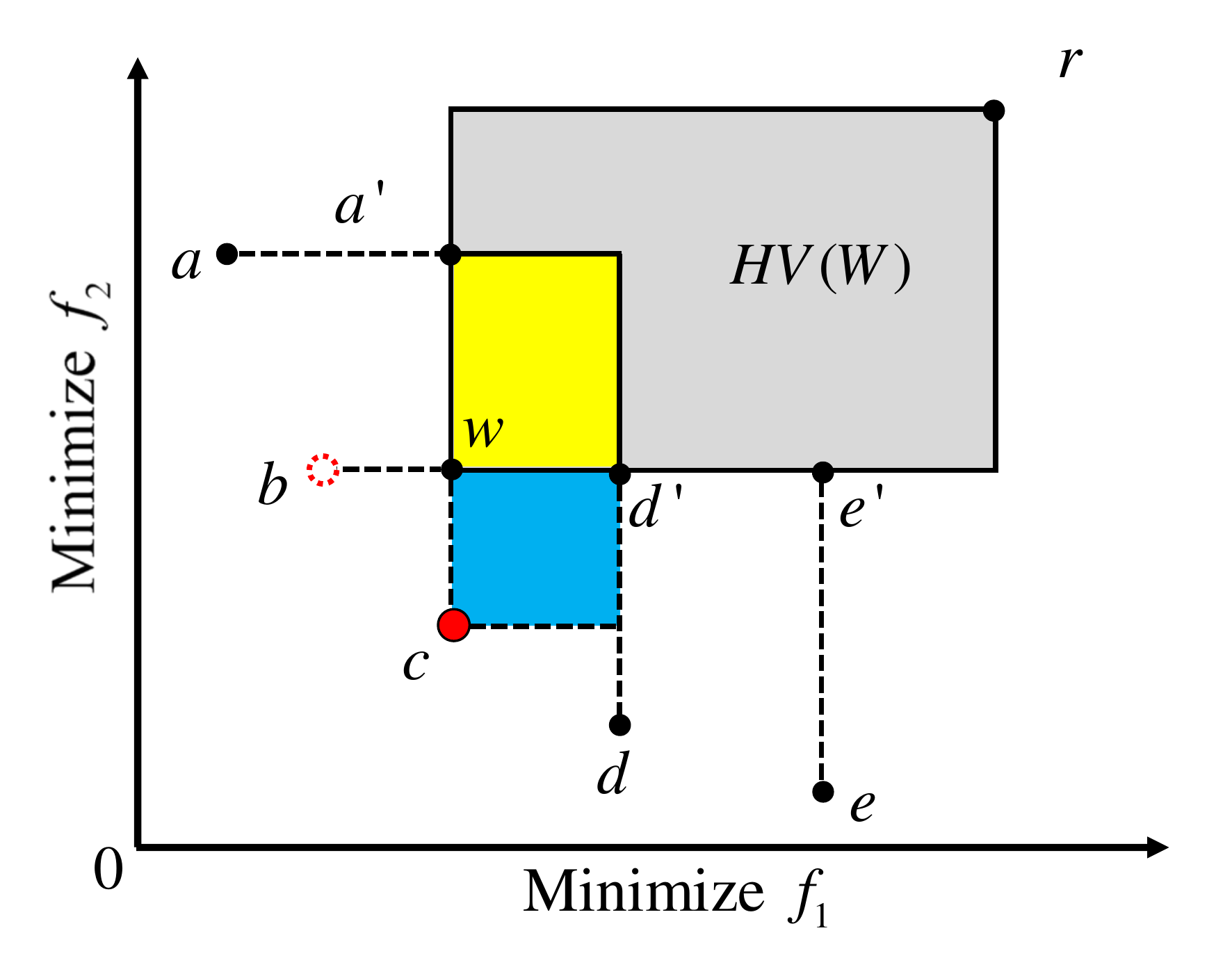}
\caption{Illustration of the hypervolume contribution update method in FV-MOEA. In this figure, it is assumed that point \(b\) has just been removed and the hypervolume contribution of point \(c\) is to be updated. }
\end{figure}
\section{IGD and IGD+ Subset Selection}
The Inverted Generational Distance (IGD) indicator \cite{32Coello2004A} is another widely-used indicator to evaluate solution sets. It calculates the average Euclidean distance from each reference point to the nearest solution. Formally, the IGD of a solution set \(S\) with respect to a reference set \(R\) can be calculated using the following formula:
\begin{equation}
IGD(S,R) = \frac{1}{|R|}\sum_{r \in R} \min_{s \in S} \sqrt{\sum_{i = 1}^d (s_i-r_i)^2} .
\end{equation}

The IGD indicator can evaluate the convergence and diversity of a solution set in polynomial time complexity. However, one of the disadvantages of the IGD indicator is that it is not Pareto compliant \cite{33Ishibuchi}. This means that a solution set \(A\) which dominates another solution set \(B\) may have a worse IGD value than \(B\). 
The Inverted Generational Distance plus (IGD+) indicator \cite{33Ishibuchi} overcomes this disadvantage. Instead of using Euclidean distance to calculate the difference between a solution and a reference point, the IGD+ indicator uses a different distance called IGD+ distance. It is based on the dominance relation. The formula to calculate the IGD+ distance between a solution \(s\) and a reference point \(r\) is as follows (for minimization problems):
\begin{equation}
D^+ (s,r)=\sqrt{\sum_{i = 1}^m(max\{ 0, s_i - r_i\})^2}.
\end{equation}

Based on the definition of the IGD+ distance, the formula to calculate the IGD+ value of a solution set $S$ with respect to a reference set \(R\) is:
\begin{equation}
IGD^+(S,R) = \frac{1}{|R|}\sum_{r \in R} \min_{s \in S} \sqrt{\sum_{i = 1}^m(max\{ 0, s_i - r_i\})^2} .
\end{equation}

Fig. 4 illustrates the calculation of IGD and IGD+ of a solution set. In this example, the solution set \(S\) has three solutions \(a\), \(b\) and \(c\) (i.e., red points in the figure). The reference set \(R\) has four points $\alpha, \beta, \gamma, \delta $ (i.e., blue points in the figure). To calculate the IGD value of \(S\), firstly, we need to find the minimum Euclidean distance from each reference point in \(R\) to the solutions in the solution set \(S\). They are shown as black lines in Fig. 4). Then, the IGD of \(S\) is the mean value of these distances. 
To calculate IGD+ of \(S\) in Fig. 4, we should calculate the IGD+ distance between each reference point and each solution. Since solution \(b\) is dominated by the reference point \(\beta\), \(b_i-\beta_i\) is large or equal to zero for all objectives. Therefore, the IGD+ distance from solution \(\beta\) to \(b\) is the same as the Euclidean distance. For solution \(a\) and reference point \(\alpha\), since \(a_2-\alpha_2\) is smaller than zero, we only calculate the positive part (i.e., \(a_1-\alpha_1\)). Thus, the IGD+ distance between $a$ and $\alpha$ is the green dash line between $a$ and $\alpha$. In a similar way, we can obtain the IGD+ distance from each reference point to each solution and find the minimum IGD+ distance from each reference point to the solution set \(S\) (i.e., green dash lines in Fig. 4). The IGD+ of $S$ is the mean value of these distances.
\begin{figure}[htbp]
\centering
\includegraphics[width=0.35 \textwidth, trim=0 13 0 0,clip]{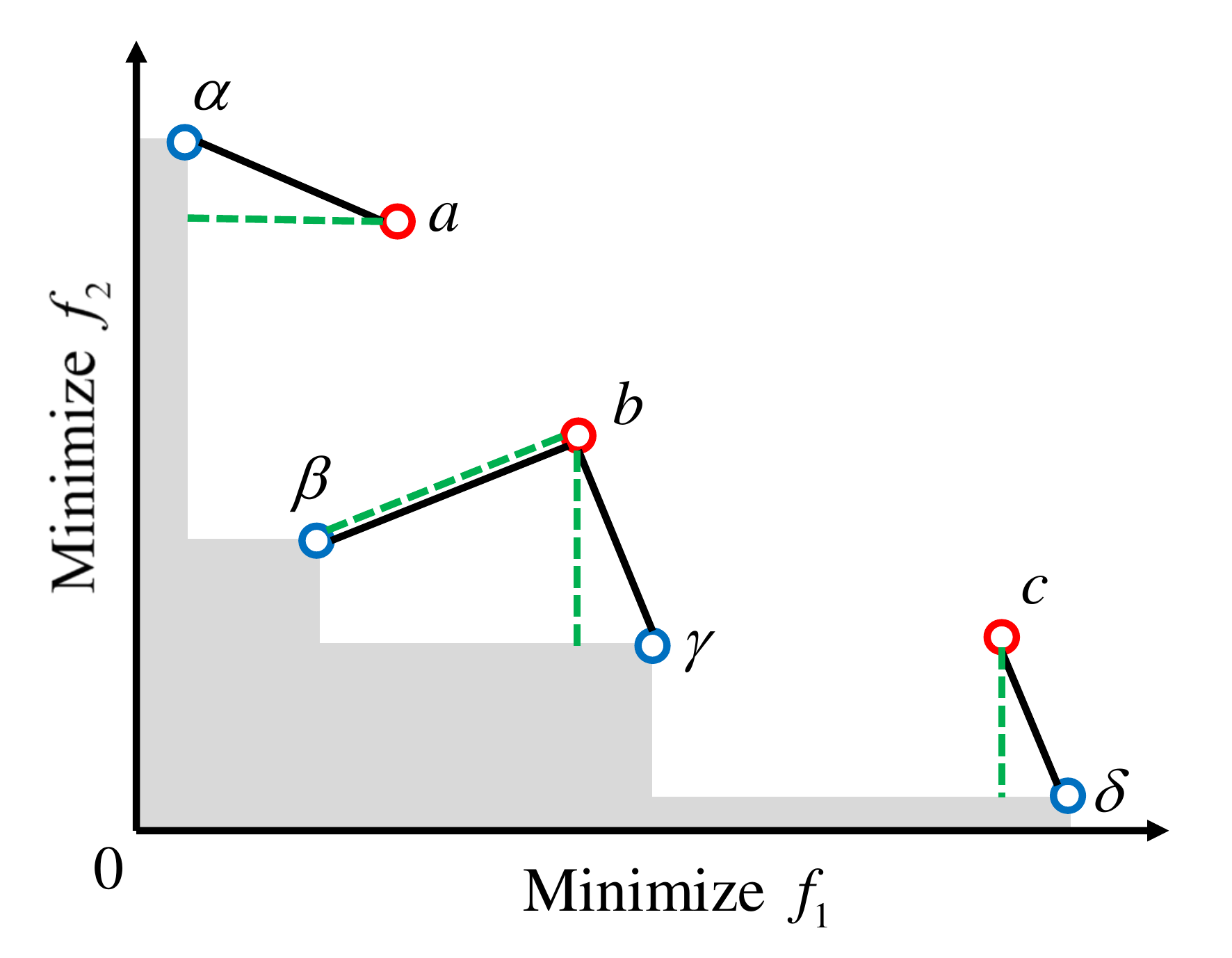}
\caption{Illustration of the calculation of IGD and IGD+.}
\end{figure}

\subsection{IGD/IGD+ Subset Selection Problem}
In contrast to hypervolume, the smaller IGD or IGD+ value indicates a better solution set. Therefore, the IGD subset selection problem takes the minus IGD of a solution set as the performance indicator \(g\) in \textbf{Definition 1} while the IGD+ subset selection problem takes the minus IGD+ of a solution set. Similar to HSSP, it is impractical to obtain the optimal subset unless the candidate solution set size is small. 

Although IGD is a popular indicator to evaluate the performance of an EMO algorithm, the IGD subset selection problem is seldom studied. The reason is that calculating the IGD needs a reference point set \(R\). When evaluating an algorithm on an artificial test problem, we can use the points on the true PF of the test problem to evaluate the solution set (since the true PF is known). However, in general, the true PF of the problem is unknown, which brings difficulties to the calculation of IGD \cite{44IGDShape}. 

One approach is to generate reference points from the estimated PF as in IGD-based EMO algorithms \cite{34Lopez,35IGD}. Firstly, the extreme points in the current solution set are found. Next, based on these extreme points, we specify a hyperplane in the objective space. Then, we uniformly sample reference points from the hyperplane for IGD/IGD+ calculation.
  
Another approach is to use the whole candidate solution set as the reference point set. In EMO algorithms with a bounded or unbounded external archive, some or all non-dominated solutions among the evaluated solutions during their execution are stored externally. In this paper, we assume such an EMO algorithm, and we use all the stored non-dominated solutions in the external archive as the reference point set for IGD/IGD+ calculation. As shown in Table III in Section V for some real-world problems, tens of thousands of non-dominated solutions are found by a single run of an EMO algorithm. However, it is often unrealistic to show such a large number of solutions to the decision maker. Thus, subset selection is used to choose a pre-specified number of solutions. Subset selection is also needed in performance comparison of EMO algorithms to compare them using solution sets of the same size. In the proposed IGD/IGD+ subset selection algorithms, all the obtained non-dominated solutions (i.e., all candidate solutions) are used as the reference point set.

The following are two advantages of IGD and IGD+ subset selection over hypervolume subset selection: (i) When the number of dimensions is large, calculating hypervolume is very time-consuming since hypervolume calculation is \#P-hard in the number of objectives \cite{19bringmann2010approximating}. However, the increase in the computation load for IGD and IGD+ is linear with respect to the number of objectives. (ii) The behavior of the hypervolume indicator is difficult to explain. Choosing different reference points can result in totally different optimal distributions \cite{36HowToReference}. The optimal distribution of solutions on PFs with different shapes can have a significant difference. Besides, the hypervolume optimal $\mu$ distributions in three or higher dimensions are still unknown. In contrast to hypervolume, IGD and IGD+ subset selection can be formulated as a problem to minimize the expected loss function \cite{17Ishibuchi}. From this point of view, IGD subset selection and IGD+ subset selection have a clear meaning for the decision-maker.

\subsection{IGD/IGD+ Subset Selection Algorithms}
Similar to HSSP, we can use greedy inclusion algorithms, greedy removal algorithms and evolutionary algorithms to find an approximate solution to the IGD and IGD+ subset selection problems. In greedy algorithms, we always need to calculate the IGD improvement of a solution \(a\) to the current solution set \(S\) with respect to the reference set \(R\). According to the definition of IGD, the IGD improvement of solution \(a\) to \(S\) can be calculated by $IGD(S,R) – IGD(S\cup\{a\},R)$. However, this is not efficient since IGD improvement calculation needs to calculate IGD twice. The time complexity of IGD improvement calculation is \(O(m|S||R|)\) for an \(m\)-objective problem with a candidate solution set \( S\) and a reference point set \(R\).

In this paper, we use a more efficient calculation method. When calculating IGD of a solution set, we also use an array \(D\) to store the distance from each reference point to the nearest solution in the solution set \(S\). When we want to calculate the IGD improvement of a new solution \(a\) to \(S\), we only need to calculate the distance from each reference point to \(a\) and store them in a new distance array \(D’\). Then, we compare each item in \(D\) and \(D’\). If the item in \(D’\) is smaller than \(D\), we replace the item in \(D\) with the corresponding item in \(D’\). The new IGD of the solution set \(S\cup\{a\}\) is the average value of each item in the new array \( D\). Finally, we subtract the new IGD from the original IGD to obtain the IGD improvement. The proposed method can reduce the complexity of IGD improvement calculation to \(O(m|R|)\), which can significantly decrease the computation time. 

The details of the IGD greedy inclusion algorithm with the proposed efficient IGD improvement calculation are shown in Algorithm 3. In this algorithm, $mean(\cdot)$ calculates the mean value of each item in an array and $min(\cdot,\cdot)$ takes the smaller one between two items that have the same index in the two arrays. 
\begin{algorithm}
	\caption{Greedy Inclusion IGD Subset Selection}
	\begin{algorithmic}[1]
    \REQUIRE  $V$ (A set of non-dominated solutions), $k$ (Solution subset size)
    \ENSURE $S$  (The selected subset from $V$)
     \IF{$|V| < k $}
           \STATE $S = V$
     \ELSE
  			\STATE  $S = \emptyset$
            \STATE $s_{min}$ = Solution in $V \backslash S$ with the smallest $IGD(s_i, V)$
            \STATE $S = S \cup \{s_{min}\}$
            \STATE $D$ = Euclidean distance from each reference point to $s_{min}$
    		\WHILE{$|S| < k $}
    			\FOR {\textbf{each} $s_i$ in $V \setminus S$}
                        \STATE $D'$ =  Euclidean distance from each reference point to $s_i$
         				\STATE $c_i = mean(D) - mean(min(D, D'))$
                 \ENDFOR
      			\STATE $p$ = Solution in  $V\setminus S$  with the largest $c_i$
                \STATE $D'$ =  Euclidean distance from each reference point to $p$
                \STATE $D = min(D, D')$
      			\STATE $S = S \cup \{p\}$ 
             \ENDWHILE  
      \ENDIF
	\end{algorithmic}
\end{algorithm}
Algorithm 3 is for IGD subset selection. It can be easily changed to the algorithm for IGD+ subset selection by replacing the Euclidean distance with the IGD+ distance. Besides, we can also use the efficient IGD improvement calculation method in IGD and IGD+ greedy removal algorithms.

\section{The Proposed Lazy Greedy Algorithms}
\subsection{Submodularity of Indicators}
The core idea of the proposed algorithms is to exploit the submodular property of the hypervolume, IGD and IGD+ indicators. Submodular function is a kind of set function that satisfies diminishing returns property. Formally, the definition of a submodular function is as follows \cite{6Nemhauser}.
\begin{definition}[Submodular Function]
A real-valued function z(V) defined on the set of all subsets of V that satisfies\\ 
$z(S_1\cup\{p\})-z(S_1) \leq z(S_2\cup \{p\})-z(S_2)$, $S_2 \subset S_1 \subset V, \\p \in V-S_1$ is called a submodular set function.
\end{definition}

Note that the submodular property is different from the non-decreasing property of set functions. To further illustrate the difference between them, we show three types of set functions in Fig. 5. The hypervolume indicator is similar to \(z_2\), which is non-decreasing (i.e., \(HV(Y) \geq HV(X), \text{if }  X \subseteq Y\)) and submodular (i.e., \(HVC(p, X) \geq HVC(p, Y), \text{if } X \subseteq Y\)). \(z_1\) is submodular but it is not non-decreasing. \(z_3\) is non-decreasing but it is non-submodular. 

\begin{figure}[htbp]
\centering
\includegraphics[width=0.3 \textwidth, trim=10 20 0 0,clip]{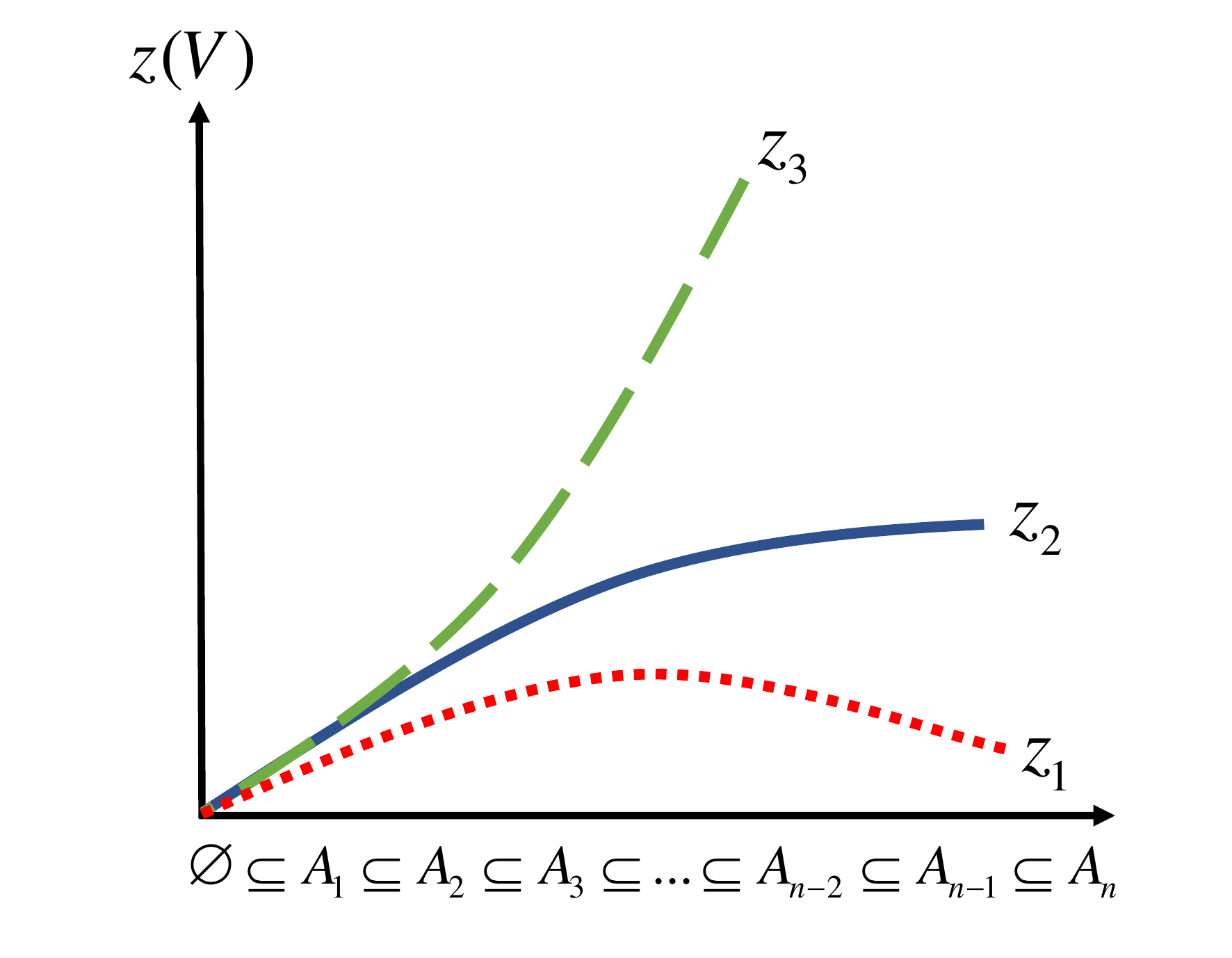}
\caption{Illustration of three types of set functions.}
\end{figure}

The hypervolume indicator has been proved to be non-decreasing submodular \cite{37ulrich2012bounding}. Now, we show that the minus IGD indicator and the minus IGD+ indicator are both non-decreasing submodular. The minus here means to take the opposite number of the indicator value. 
\begin{theorem}
The minus IGD indicator is non-decreasing submodular.
\end{theorem}
\noindent \begin{proof}
Let us consider a solution set \(A\), a reference set \(R\) and two solutions \(x\) and \(y\) such that \(x \notin A\) and \(y \notin A\). By adding a solution \(x\) to \(A\), the nearest distance from some reference point to $A$ will be changed. Let us denote the set of indexes of these points as \( I_1\), and the difference between the original distance and the new distance as \(\Delta d_i^x, i \in I_1\). Similarly, by adding a solution \(y\), the set of indexes of reference points whose nearest distance are changed is denoted as  \(I_2\), and the corresponding distance differences are denoted as \(\Delta d_i^y, i \in I_2\). The notations used in the proof are summaries in Table 1.

\begin{table}[]
\caption{Summary of Notations}
\begin{tabular}{p{40pt}p{180pt}}
\toprule
\textbf{Notation} & \textbf{Description} \\ \midrule
$A, B, X$   &    Solution sets  \\
$R$   &    Reference set  \\
$x, y$   &  Solutions that do not belong to $A$     \\
$i$   &  Index of the reference point in the reference set   \\
$I_1, I_2$        &    Sets of reference point indexes whose closest distance to the solution set are changed after adding $x$ and $y$, respectively   \\
$\Delta d_i^x$, $\Delta d_i^y$     &  Change of the closest distance between reference point $i$ and the solution set by adding $x$ and $y$, respectively   \\ \bottomrule
\end{tabular}
\end{table}

According to the definition of IGD, $IGD(A\cup \{x\}) = IGD(A)-  \frac{1}{|R|} \sum_{i \in I_1}\Delta d_i^x.$ Since $\Delta d_i^x$  is larger than zero, the IGD indicator is non-increasing. Therefore, the minus IGD indicator is non-decreasing.

	Similarly, by adding a solution $y$ to a solution set $A\cup\{x\}$,  
	\begin{align*}
	    IGD(A\cup\{x\}\cup\{y\}) = & IGD(A) - \frac{1}{|R|}\sum_{i \in I_1 \backslash (I_1 \cap I_2)} \Delta d_i^x \\ & - \frac{1}{|R|}\sum_{i \in I_2\backslash(I_1 \cap I_2)} \Delta d_i^y \\ &- \frac{1}{|R|}\sum_{i \in (I_1 \cap I_2)} \max\{\Delta d_i^x, \Delta d_i^y\}.
	\end{align*}

    In
    \begin{align*}
      & \frac{1}{|R|}\sum_{i \in I_1 \backslash (I_1 \cap I_2)} \Delta d_i^x + \frac{1}{|R|}\sum_{i \in I_2\backslash(I_1 \cap I_2)} \Delta d_i^y \\ & + \frac{1}{|R|}\sum_{i \in (I_1 \cap I_2)} \max\{\Delta d_i^x, \Delta d_i^y\},  
    \end{align*}
     if $i \in (I_1 \cap I_2)$, only the larger one between $\Delta d_i^x$ and $\Delta d_i^y$  will be added while in $\sum_{i \in I_1} \Delta d_i^x + \sum_{i \in I_2} \Delta d_i^y$  all $\Delta d_i^x$ and $\Delta d_i^y$ will be added together. 
     
     Hence,
     \begin{align*}
 &\frac{1}{|R|}\sum_{i \in I_1 \backslash (I_1 \cap I_2)} \Delta d_i^x + \frac{1}{|R|}\sum_{i \in I_2\backslash(I_1 \cap I_2)} \Delta d_i^y \\&+ \frac{1}{|R|}\sum_{i \in (I_1 \cap I_2)} \max\{\Delta d_i^x, \Delta d_i^y\} \leq \sum_{i \in I_1} \Delta d_i^x + \sum_{i \in I_2} \Delta d_i^y.\\
& -IGD(A\cup \{y\}) - (-IGD(A)) = \frac{1}{|R|}\sum_{i \in I_2}\Delta d_i^y.
\end{align*}   

\begin{align*}
    & IGD(A\cup \{x\} \cup \{y\}) - (-IGD(A\cup\{x\})) \\& = \frac{1}{|R|}\sum_{i \in I_1 \backslash (I_1 \cap I_2)} \Delta d_i^x + \frac{1}{|R|}\sum_{i \in I_2\backslash(I_1 \cap I_2)} \Delta d_i^y \\& + \frac{1}{|R|}\sum_{i \in (I_1 \cap I_2)} \max\{\Delta d_i^x, \Delta d_i^y\} - \frac{1}{|R|}\sum_{i \in I_2}\Delta d_i^x \\& \leq  \frac{1}{|R|}\sum_{i \in I_1} \Delta d_i^x +  \frac{1}{|R|}\sum_{i \in I_2} \Delta d_i^y - \frac{1}{|R|}\sum_{i \in I_1} \Delta d_i^x \\& = \frac{1}{|R|}\sum_{i \in I_2}\Delta d_i^y.
\end{align*}

Therefore, $ -IGD(A\cup \{y\}) - (-IGD(A)) \geq IGD(A\cup \{x\} \cup \{y\}) - (-IGD(A\cup\{x\})).$

Then, by mathematical induction, we can obtain that

$ -IGD(A\cup \{y\}) - (-IGD(A)) \geq IGD(A\cup X \cup \{y\}) - (-IGD(A\cup X)),$where $X$ is a solution set and $X \cap A = \emptyset$. If we let $B = A \cup X$, the above formula can be rewritten as 
$ -IGD(A\cup \{y\}) - (-IGD(A))  \geq IGD(B \cup \{y\}) - (-IGD(B)), A \subset B,$
which is the same as the definition of the submodular function. Hence, the minus IGD indicator is non-decreasing submodular. 
\end{proof}
\begin{theorem}
The minus IGD+ indicator is non-decreasing submodular.
\end{theorem}
\begin{proof}
We can prove that the minus IGD+ indicator is also non-decreasing submodular in the same manner as the proof for the submodularity of IGD indicator. The only difference is to replace the Euclidean distance with the IGD+ distance. 
\end{proof}

\subsection{Algorithm Proposal}
In each iteration of the hypervolume greedy inclusion algorithm, we only need to identify the solution with the largest hypervolume contribution. However, we usually calculate the hypervolume contributions of all solutions. Since the calculation of the contribution of each solution is time-consuming, such an algorithm is not efficient. 

As discussed in the last subsection, the hypervolume indicator is non-decreasing submodular. This property can help us avoid unnecessary calculations. The submodular property of the hypervolume indicator means that the hypervolume contribution of a solution to the selected solution subset S never increases as the number of solutions in $S$ increases in a greedy inclusion manner. Hence, instead of recomputing the hypervolume contribution of every candidate solution in each iteration, we can utilize the following lazy evaluation mechanism. 

We use a list \(C\) to store the candidate (i.e., unselected) solutions and their tentative HVC (hypervolume contribution) values. The tentative HVC value of each solution is initialized with its hypervolume (i.e., its hypervolume contribution when no solution is selected). The tentative HVC value of each solution is the upper bound of its true hypervolume contribution. For finding the solution with the largest hypervolume contribution from the list, we pick the most promising solution with the largest tentative HVC value and recalculate its hypervolume contribution to the current solution subset \(S\). If the recalculated hypervolume contribution of this solution is still the largest in the list, we do not have to calculate the hypervolume contributions of the other solutions. This is because the hypervolume contribution of each solution never increases through the execution of greedy inclusion. In this case (i.e., if the recalculated hypervolume contribution of the most promising solution is still the largest in the list), we move this solution from the list to the selected solution subset \(S\). If the recalculated hypervolume contribution of this solution is not the largest in the list, its tentative HVC value is updated with the recalculated value. Then the most promising solution with the largest tentative HVC value in the list is examined (i.e., its hypervolume contribution is recalculated). This procedure is iterated until the recalculated hypervolume contribution is the largest in the list.

In many cases, the recalculation of the hypervolume contribution of each solution results in the same value as or slightly smaller value than its tentative HVC value in the list since the inclusion of a single solution to the solution subset $S$ changes the hypervolume contributions of only its neighbors in the objective space. Thus, the solution with the largest hypervolume contribution is often found without examining all solutions in the list. By applying this lazy evaluation mechanism, we can avoid a lot of unnecessary calculations in hypervolume-based greedy inclusion algorithms.
The details of the proposed lazy greedy inclusion hypervolume subset selection (LGI-HSS) algorithm are shown in Algorithm 4.

\begin{algorithm}
	\caption{Lazy Greedy Inclusion Hypervolume Subset Selection (LGI-HSS)}
	\begin{algorithmic}[1]
    \REQUIRE  $V$ (A set of non-dominated solutions), $k$ (Solution subset size)
    \ENSURE $S$  (The selected subset from $V$)
     \IF{$|V| < k $}
           \STATE $S = V$
     \ELSE
  			\STATE  $S = \emptyset$,  $C = \emptyset$
  			\FOR {\textbf{each} $s_i$ in $V$}
  			    \STATE Insert $\left(s_i, HV(\{s_i\})\right)$ into $C$
  			\ENDFOR
    		\WHILE{$|S| < k $}
    		    \WHILE{$C \not= \emptyset$}
    		    \STATE $c_{max}$ = Solution with the largest HVC in $C$
    		    \STATE Update the HVC of $c_{max}$ to $S$
    		    \IF{$c_{max}$ has the largest HVC in $C$}
    		        \STATE $S = S \cup \{c_{max}\}$
    		        \STATE $C = C \setminus \{c_{max}\}$
    		        \STATE \textbf{break}
    		    \ENDIF
    		    \ENDWHILE
             \ENDWHILE  
      \ENDIF
	\end{algorithmic}
\end{algorithm}

For the IGD and IGD+ indicators, the same idea can be used. Since the submodularity of the IGD and IGD+ indicators has been proved in section IV.A, we can obtain that the IGD (IGD+) improvement of a solution to the selected solution set $S$ will never increase. Thus, we do not need to calculate other solutions if the recalculated IGD (IGD+) improvement of a solution is the largest among all tentative IGD (IGD+) improvement values. Besides, to accelerate the IGD (IGD+) improvement calculation, we use our proposal for the standard IGD greedy inclusion in Section III. An array \(D\) is used to store the distance from each reference point to the nearest solution in the solution set \(S\). The details of the lazy greedy inclusion IGD subset selection algorithm (LGI-IGDSS) is shown in Algorithm 5. It can be changed to the algorithm for the IGD+ indicator (LGI-IGD+SS) by using the IGD+ distance instead of the Euclidean distance to calculate the distances between each reference point and each solution.

Note that we need to find the solution with the largest hypervolume contribution in Algorithm 4 and the largest IGD (IGD+) improvement in Algorithm 5. The priority queue implemented by the maximum heap is used to accelerate the procedure.

\begin{algorithm}
	\caption{Lazy Greedy Inclusion IGD Subset Selection (LGI-IGDSS)}
	\begin{algorithmic}[1]
    \REQUIRE  $V$ (A set of non-dominated solutions), $k$ (Solution subset size)
    \ENSURE $S$  (The selected subset from $V$)
     \IF{$|V| <  k $}
           \STATE $S = V$
     \ELSE
  			\STATE  $S = \emptyset$,  $C = \emptyset$
            \STATE $s_{min}$ = Solution in $V\backslash S$ with the smallest $IGD(\{s_i\}, V)$
            \STATE $S = S \cup \{s_{min}\}$
            \STATE \(D'\) = Distance from each reference point to $s_{min}$

  			\FOR {\textbf{each} $s_i$ in $V\setminus S$}
  			    \STATE Insert $\left(s_i, IGD(\{s_{min}\}, R)- IGD(\{s_{min}, s_{i}\}, R)\right)$ into $C$
  			\ENDFOR
    		\WHILE{$|S| < k $}
    		    \WHILE{$C \not= \emptyset$}
    		    \STATE $c_{max}$ = Solution with the largest IGD improvement in $C$
                \STATE Distance from each reference point to $c_{max}$
    		    \STATE Update the IGD improvement of $c_{max}$ to $mean(D') - mean(min(D', D))$
    		    \IF{$c_{max}$ has the largest IGD improvement in $C$}
    		        \STATE $S = S \cup \{c_{max}\}$
    		        \STATE $C = C \setminus \{c_{max}\}$
                    \STATE $D$ = Distance from each reference point to $c_{max}$
                    \STATE $D = min(D', D)$
    		        \STATE \textbf{break}
    		    \ENDIF
    		    \ENDWHILE
             \ENDWHILE  
      \ENDIF
	\end{algorithmic}
\end{algorithm}

The proposed algorithms only need one parameter $k$, which is the solution subset size. This parameter is needed in all subset selection algorithms. That is, the proposed algorithms do not need any additional parameter. 

The idea of the lazy evaluation was proposed by Minoux \cite{38minoux1978accelerated} to accelerate the greedy algorithm for maximizing submodular functions. Then, it was applied to some specific areas such as influence maximization problems \cite{39leskovec2007cost} and network monitoring \cite{47NetworkInference}. Minoux \cite{38minoux1978accelerated} proved that if the function is submodular and the greedy solution is unique, the solution produced by the lazy greedy algorithm and the original greedy algorithm is identical. Since the hypervolume, IGD and IGD+ indicators are submodular, the proposed algorithms (i.e., Algorithm 4 and Algorithm 5) find the same subsets as the corresponding original greedy inclusion algorithms if the greedy solutions are unique. 

\subsection{An Illustrative Example}
Let us explain the proposed algorithm using a simple example. Fig. 6 shows the changes of the hypervolume contribution in the list \(C\). The value in the parentheses is the stored HVC value of each solution to the selected subset. For illustration purposes, the solutions in the list are sorted by the stored HVC values. However, in the actual implementation of the algorithm, the sorting is not necessarily needed (especially when the number of candidate solutions is huge). This is because our algorithm only needs to find the most promising candidate solution with the largest HVC value in the list. 
Fig. 6 (i) shows the initial list \(C\) including five solutions \(a\), \(b\), \(c\), \(d\) and \(e\). The current solution subset is empty. In Fig. 6 (i), solution a has the largest HVC value. Since the initial HVC value of each solution is the true hypervolume contribution to the current empty solution subset \(S\), no recalculation is needed. Solution \(a\) is moved from the list to the solution subset.

In Fig. 6 (ii), solution \(b\) has the largest HVC value in the list after solution \(a\) is moved. Thus, the hypervolume contribution of \(b\) is to be recalculated. We assume that the recalculated HVC value is 4 as shown in Fig. 6 (iii).

Fig. 6 (iii) shows the list after the recalculation. Since the updated HVC value of \(b\) is not the largest, we need to choose solution \(e\) with the largest HVC value in the list and recalculate its hypervolume contribution. We assume that the recalculated HVC value is 6 as shown in Fig. 6 (iv).

Fig. 6 (iv) shows the list after the recalculation. Since the recalculated HVC value is still the largest in the list, solution \(e\) is moved from the list to the solution subset.

Fig. 6 (v) shows the list after the removal of \(e\). Solution \(c\) with the largest HVC value is examined.

In this example, for choosing the second solution from the remaining four candidates (\(b\), \(c\), \(d\) and \(e\)), we evaluate the hypervolume contributions of only the two solutions (\(b\) and \(e\)). In the standard greedy inclusion algorithm, all four candidates are examined. In this manner, the proposed algorithm decreases the computation time of the standard greedy inclusion algorithm.
\begin{figure}[htbp]
\centering
\includegraphics[width=0.36\textwidth, trim = 0 100 0 20, clip]{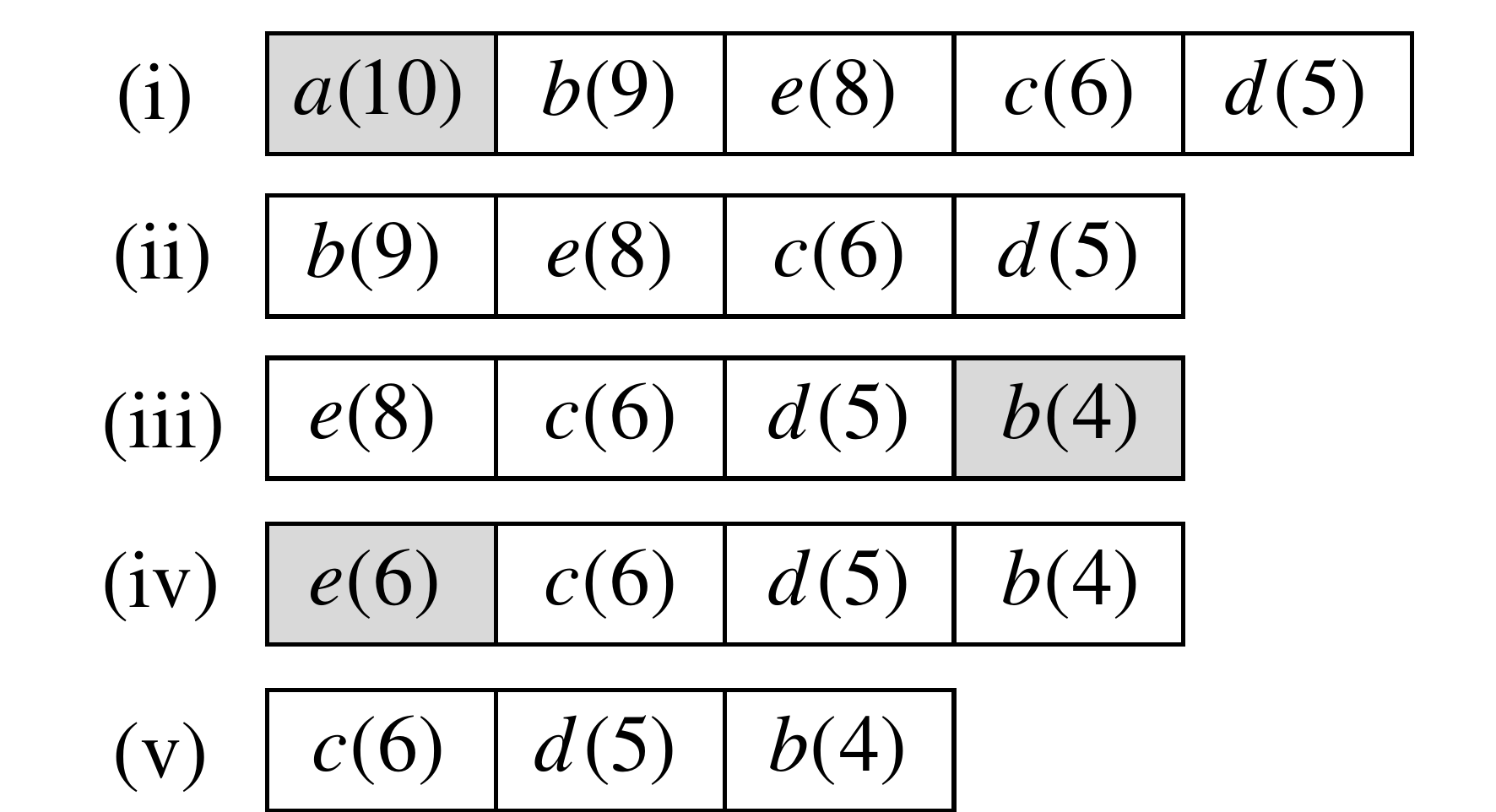}
\caption{Illustration of the proposed algorithm. The values in the parentheses are the stored tentative HVC values.}
\end{figure}

\section{Experiments}
\subsection{Algorithms Used in Computational Experiments }
The proposed lazy greedy inclusion hypervolume subset selection algorithm LGI-HSS is compared with the following two algorithms:
\subsubsection{Standard greedy inclusion hypervolume subset selection (GI-HSS)} This is the greedy inclusion algorithm described in Section II.C. When calculating the hypervolume contribution, the efficient method described in Section II.A is employed.
\subsubsection{Greedy inclusion hypervolume subset selection with hypervolume contribution updating (UGI-HSS)} The hypervolume contribution updating method proposed in FV-MOEA \cite{31jiang2014simple} (Algorithm 2) is used. Since Algorithm 2 is for greedy removal, it is changed for greedy inclusion here. It is the fastest known greedy inclusion algorithm applicable to any dimension. 

The proposed lazy greedy inclusion IGD subset selection algorithm LGI-IGDSS is compared with the standard greedy inclusion algorithm GI-IGDSS with the efficient IGD improvement evaluation method in Section III.B. The proposed lazy greedy inclusion IGD+ subset selection algorithm LGI-IGD+SS is compared with the standard greedy inclusion algorithm GI-IGD+SS with the efficient IGD+ improvement evaluation method in Section III.B.

Our main focus is the selection of a solution subset from an unbounded external archive. Since the number of solutions to be selected is much smaller than the number of candidate solutions: $k \ll n = |V|$ in HSSP, greedy removal is not efficient. Hence, greedy removal algorithms (e.g., those with IHSO* \cite{29Bradstreet2008fast} and IWFG \cite{30cox2016improving}) are not compared in this paper.
\subsection{Test Problems and Candidate Solution Sets}
To examine the performance of the above-mentioned subset selection algorithms, we use the following four representative test problems with different Pareto front shapes:

(i)	DTLZ2 \cite{41DTLZ} with a triangular concave (spherical) Pareto front. 

(ii) Inverted DTLZ2 (I-DTLZ2 \cite{42IDTLZ}) with an inverted triangular concave (spherical) Pareto front.

(iii) DTLZ1 \cite{41DTLZ} with a triangular linear Pareto front.

(iv) DTLZ7 \cite{41DTLZ} with a disconnected Pareto front.

For each test problem, we use three problem instances with five, eight and ten objectives (i.e., solution subset selection is performed in five-, eight- and ten-dimensional objective spaces). Candidate solution sets are generated by sampling solutions from the Pareto front of each problem instance. Four different settings of the candidate solution set size are examined: 5000, 10000, 15000 and 20000 solutions. First we uniformly sample 100,000 solutions from the Pareto front of each problem instance. In each run of a solution subset selection algorithm, a pre-specified number of candidate solutions (i.e., 5000, 10000, 15000 or 20000 solutions) are randomly selected from the generated 100,000 solutions. Computational experiments are performed 11 times for each setting of the candidate solution set size for each problem instance by each subset selection algorithm. The number of solutions to be selected is specified as 100. Thus, our problem is to select 100 solutions from 5000, 10000, 15000 and 20000 candidate solutions to maximize the hypervolume value, and to minimize the IGD and IGD+ values of the selected 100 solutions.
\subsection{Experimental Settings}
In each hypervolume subset selection algorithm, the reference point for hypervolume (contribution) calculation is specified as (1.1, 1.1, ..., 1.1) for all test problems independent of the number of objectives. We use the WFG algorithm\footnote{The code of the WFG algorithm is available from http://www.wfg.csse.uwa.edu.au/hypervolume/\#code.}  \cite{27while2011fast} for hypervolume calculation in each subset selection algorithm with the hypervolume indicator. All subset selection algorithms are coded by MatlabR2018a. The computation time of each run is measured on an Intel Core i7-8700K CPU with 16GB of RAM, running in Windows 10.
\subsection{Comparison of Hypervolume Subset Selection Algorithms}
The average run time of each hypervolume subset selection algorithm on each test problem is summarized in Table II (in the columns labelled as GI-HSS, UGI-HSS and LGI-HSS). Since the subset obtained by three algorithms on the same data set are the same (i.e., the proposed algorithm does not change the performance),  we only show the run time of each algorithm.

Compared with the standard GI-HSS algorithm, we can see that our LGI-HSS algorithm can reduce the computation time by 91\% to 99\%. That is, the computation time of our LGI-HSS is only 1-9\% of that of GI-HSS. By the increase in the number of objectives (i.e., by the increase in the dimensionality of the objective space), the advantage of LGI-HSS over the other two algorithms becomes larger. Among the four test problems in Table II, all the three algorithms are fast on the I-DTLZ2 problem and slow on the DTLZ2 and DTLZ1 problems.
\begin{table*}[]
\centering
\caption{Average Computation Time (in Seconds) on Different Problems Over 11 Runs. The Best Results are Highlighted by Bold. }
\label{table1}
\begin{tabular}{cc|ccc|cc|cc}
\toprule
Problem                                  & Candidate Solutions & GI-HSS  & UGI-HSS & LGI-HSS & GI-IGDSS & LGI-IGDSS & GI-IGD+SS & LGI-IGD+SS \\ \hline
\multirow{4}{*}{Five-Objective DTLZ2}   & 5,000       & 46.9    & 10.7    & \textbf{4.3}     & 54.9     & \textbf{4.5}       & 246.2     & \textbf{20.9}       \\
                                         & 10,000      & 100.2   & 21.5    & \textbf{8.9}     & 137.5    & \textbf{10.6}      & 980.5     & \textbf{81.8}       \\
                                         & 15,000      & 164.7   & 31.2    & \textbf{14.5}    & 248.2    & \textbf{18.6}      & 2197.5    & \textbf{179.8}      \\
                                         & 20,000      & 243.2   & 40.7    & \textbf{21.1}    & 398.5    & \textbf{29.1}      & 3866.0    & \textbf{311.6}      \\ \hline
\multirow{4}{*}{Eight-Objective DTLZ2}   & 5,000       & \textbf{}413.8   & \textbf{}75.4    & \textbf{18.8}    & 53.4     & \textbf{4.6}       & 253.4     & \textbf{19.5}       \\
                                         & 10,000      & 726.0   & 132.7   & \textbf{32.5}    & 149.2    & \textbf{12.0}      & 1015.4    & \textbf{74.4}       \\\textbf{}
                                         & 15,000      & 1015.2  & 188.3   & \textbf{47.3}    & 284.7    & \textbf{21.8}      & 2294.9    & \textbf{165.7}      \\
                                         & 20,000      & 1251.2  & 245.6   & \textbf{61.3}    & 1440.1   & \textbf{102.8}     & 4088.7    & \textbf{292.4}      \\ \hline
\multirow{4}{*}{Ten-Objective DTLZ2}     & 5,000       & 3521.7  & 540.6   &\textbf{} \textbf{138.1}   & 55.8     & \textbf{5.1}       & 263.9     & \textbf{20.1}       \\
                                         & 10,000      & 5569.6  & 830.8   & \textbf{217.4}   & 159.0    & \textbf{14.2 }     & 1054.7    & \textbf{76.8}       \\
                                         & 15,000      & 7646.2  & 1277.6  & \textbf{291.8}   & 995.2    & \textbf{83.0}      & 2356.0    & \textbf{171.4}      \\
                                         & 20,000      & 9116.0  & 1552.1  & \textbf{357.5}   & 1665.6   & \textbf{124.5}     & 4175.9    & \textbf{297.2}      \\ \hline
\multirow{4}{*}{Five-Objective I-DTLZ2} & 5,000       & 35.7    & 4.1     & \textbf{1.8}     & 52.7     & \textbf{4.3}       & 247.3     & \textbf{18.2}       \\
                                         & 10,000      & 82.7    & 8.2     & \textbf{4.0}     & 137.9    & \textbf{10.7}      & 979.5     & \textbf{71.1}       \\
                                         & 15,000      & 140.4   & 12.0    & \textbf{6.9}     & 245.9    & \textbf{18.3}      & 2190.2    & \textbf{155.0}      \\
                                         & 20,000      & 219.8   & 15.5    & \textbf{10.4}    & 410.4    & \textbf{30.5}      & 3890.5    & \textbf{276.7}      \\ \hline
\multirow{4}{*}{Eight-Objective I-DTLZ2} & 5,000       & 46.0    & 5.4     & \textbf{0.8}     & 59.0     & 5.0       & 255.7     & \textbf{17.1}       \\\textbf{}
                                         & 10,000      & 108.5   & 11.2    & \textbf{1.7}     & 158.9    & \textbf{12.7}      & 1017.6    & \textbf{65.7}       \\
                                         & 15,000      & 190.2   & 18.7    & \textbf{2.9}     & 302.3    & \textbf{22.8}      & 2302.6    & \textbf{145.5}      \\
                                         & 20,000      & 283.8   & 24.3    & \textbf{4.2}     & 1649.3   & \textbf{115.3}     & 4070.0    & \textbf{254.5}      \\ \hline
\multirow{4}{*}{Ten-Objective I-DTLZ2}   & 5,000       & 43.9    & 8.4     & \textbf{0.5}     & 55.7     & \textbf{5.2}       & 263.4     & \textbf{17.9}       \\
                                         & 10,000      & 108.2   & 14.9    & \textbf{1.2}     & 157.5    & \textbf{14.1}      & 1047.2    & \textbf{67.3}       \\
                                         & 15,000      & 195.5   & 22.0    & \textbf{2.2}     & 991.2    & \textbf{82.7}      & 2353.0    & \textbf{148.9}      \\
                                         & 20,000      & 304.1   & 33.9    & \textbf{3.3}     & 1669.7   & \textbf{135.9}     & 4162.9    & \textbf{272.8}      \\ \hline
\multirow{4}{*}{Five -Objective DTLZ1}   & 5,000       & 48.6    & 17.9    & \textbf{6.5}     & 49.3     & \textbf{3.9}       & 249.8     & \textbf{18.2}       \\
                                         & 10,000      & 108.3   & 35.6    & \textbf{14.7}    & 131.0    & \textbf{9.6}       & 985.1     & \textbf{70.8}       \\
                                         & 15,000      & 179.4   & 53.7    & \textbf{24.4}    & 247.5    & \textbf{17.6}      & 2202.8    & \textbf{154.3}      \\
                                         & 20,000      & 262.8   & 70.9    & \textbf{36.2}    & 392.4    & \textbf{27.5}      & 3886.7    & \textbf{270.3}      \\ \hline
\multirow{4}{*}{Eight-Objective DTLZ1}   & 5,000       & 362.8   & 61.6    & \textbf{35.7}    & 53.9     & \textbf{4.2}       & 255.9     & \textbf{16.2 }      \\\textbf{}
                                         & 10,000      & 777.6   & 133.9   & \textbf{79.4}    & 151.5    & \textbf{10.9}      & 1016.0    & \textbf{61.6}       \\
                                         & 15,000      & 1239.6  & 201.7   & \textbf{126.3}   & 277.9    & \textbf{19.4}      & 2282.9    & \textbf{134.4}      \\
                                         & 20,000      & 1690.3  & 285.8   & \textbf{175.5}   & 1435.8   & \textbf{92.1}      & 4070.5    & \textbf{238.6}      \\ \hline
\multirow{4}{*}{Ten-Objective DTLZ1}     & 5,000       & 2436.2  & 278.5   & \textbf{228.2}   & 55.6     & \textbf{4.3}       & 262.5     & \textbf{16.0}       \\\textbf{}
                                         & 10,000      & 5455.5  & 623.8   & \textbf{505.4}   & 158.0    & \textbf{12.2}      & 1045.2    & \textbf{65.9}       \\
                                         & 15,000      & 8504.7  & 1038.3  & \textbf{809.7}   & 995.7    & \textbf{69.3}      & 2342.7    & \textbf{141.7}      \\
                                         & 20,000      & 11552.2 & 1343.0  & \textbf{1176.0}  & 1679.3   & \textbf{118.9}     & 4144.4    & \textbf{260.1}      \\ \hline
\multirow{4}{*}{Five-Objective DTLZ7}   & 5,000       & 40.1    & 13.0    & \textbf{2.4}     & 48.5     &\textbf{ 4.3}       & \textbf{}247.7     & \textbf{20.1}       \\
                                         & 10,000      & 91.9    & 25.8    & \textbf{5.5}     & 131.5    & \textbf{11.1}      & 978.8     & \textbf{76.5}       \\
                                         & 15,000      & 155.2   & 40.4    & \textbf{9.0}     & 245.9    & \textbf{20.2}      & 2198.2    & \textbf{172.0}      \\
                                         & 20,000      & 228.5   & 56.7    & \textbf{13.3}    & 392.7    & \textbf{31.8}      & 4067.4    & \textbf{320.4 }     \\ \hline
\multirow{4}{*}{Eight-Objective DTLZ7}   & 5,000       & 52.8    & 26.5    & \textbf{1.5}     & 52.6     & \textbf{4.9}       & 269.5     & \textbf{19.3}       \\
                                         & 10,000      & 122.5   & 48.9    & \textbf{3.3}     & 147.1    & \textbf{13.1}      & 1108.5    & \textbf{69.5}       \\
                                         & 15,000      & 208.7   & 73.6    & \textbf{5.6}     & 289.1    & \textbf{24.5}      & 2281.5    & \textbf{152.0}      \\\textbf{}
                                         & 20,000      & 314.7   & 97.4    & \textbf{8.4}     & 1439.8   & \textbf{114.9}     & 4046.7    & \textbf{269.6}      \\ \hline
\multirow{4}{*}{Ten-Objective DTLZ7}     & 5,000       & 57.9    & 27.6    & \textbf{1.1}     & 55.6     & 5.4       & 262.7     & \textbf{22.8}       \\
                                         & 10,000      & 136.3   & 55.2    & \textbf{2.5 }    & 159.0    & \textbf{14.8}      & 1041.5    & \textbf{90.0}       \\
                                         & 15,000      & 237.2   & 82.3    & \textbf{4.3}     & 985.3    & \textbf{84.2}      & 2343.5    & \textbf{202.9}      \\
                                         & 20,000      & 361.7   & 110.9   & \textbf{6.5}     & 1715.5   & \textbf{143.5}     & 4144.0    & \textbf{362.3}      \\ \bottomrule   
\end{tabular}
\end{table*}

Even when we compare our LGI-HSS algorithm with the fastest known greedy inclusion algorithm UGI-HSS, LGI-HSS is also faster. On DTLZ2 and I-DTLZ2, when the number of objectives is not very large (i.e., five-objective problems), the difference in the average computation time between the two algorithms is not large (the LGI-HSS average computation time is 41-67\% of that of UGI-HSS). When the number of objectives is larger (i.e., eight-objective and ten-objective problems), the difference in the average computation time between the two algorithms becomes larger (i.e., LGI-HSS needs only 6-25\% of the UGI-HSS computation time). On DTLZ7, LGI-HSS needs only 4-24\% of the UGI-HSS computation time. On DTLZ1, the difference between LGI-HSS and UGI-HSS is small: LGI-HSS needs 36-88\% of the UGI-HSS computation time. The reason for the small difference is that each candidate solution on the linear Pareto front of DTLZ1 has a similar hypervolume contribution value. As a result, recalculation is frequently performed in our LGI-HSS. On the contrary, there exist large differences among hypervolume contribution values of candidate solutions on the nonlinear Pareto fronts of DTLZ2, I-DTLZ2 and DTLZ7. This leads to a less frequent update of their hypervolume contribution values (i.e., high efficiency of LGI-HSS).  

From the three columns by HSS in Table I, we can also observe that the average computation time of each algorithm did not severely increase with the increase in the number of objectives (i.e., with the increase in the dimensionality of the objective space) for DTLZ7 and I-DTLZ2. In some cases, the average computation time decreased with the increase in the number of objectives. The reasons are as follows. Firstly, the WFG algorithm is used in the three algorithms to calculate the hypervolume contribution of each solution. The computation time of hypervolume contribution does not increase severely as the number of objectives increases. Besides, the total number of solution evaluations needed for LGI-HSS will decrease on some problems as the number of objectives increases. This is because the difference in the hypervolume contribution values of the candidate solutions increases with the number of objectives, which leads to the decrease in the number of updates of the hypervolume contribution value of each solution.

\subsection{Comparison of IGD/IGD+ Subset Selection Algorithms}
As can be observed from the last four columns of Table II, the proposed LGI-IGDSS and LGI-IGD+SS are much faster than the standard greedy algorithms (i.e., GI-IGDSS and GI-IGD+SS). Compared with the GI-IGDSS algorithm, our LGI-IGDSS algorithm needs only 6-10\% computation time. Compared with the GI-IGD+SS algorithm, our LGI-IGD+SS algorithm needs only 6-9\% computation time. 

Different from the hypervolume subset selection algorithms, the computation time of the IGD and IGD+ subset selection algorithms does not have a large difference among different problems. One interesting observation is that the computation time of LGI-IGD+SS does not increase with the number of objectives whereas that of LGI-IGDSS clearly increases with the number of objectives. This is because the difference in the IGD+ contribution values of the candidate solutions increases with the number of objectives, which leads to the decrease in the number of updates of the IGD+ contribution value of each solution. However, the difference in the IGD contribution values of the candidate solutions does not clearly increase with the number of objectives. This observation suggests that the IGD+ indicator is more similar to the hypervolume indicator than the IGD indicator, which was suggested by the optimal distributions of solutions for each indicator \cite{46Comparision}.

\subsection{Number of Solution Evaluations}
To further examine the effectiveness of the proposed LGI-HSS, we monitor the number of solution evaluations (i.e., contribution calculations) to select each solution used by GI-HSS and LGI-HSS. In GI-HSS, all the remaining solutions need to be re-evaluated during each iteration. For example, when the candidate solution set size is 10,000, 10,000$ - (i - 1)$ solutions are examined to select the \(i\)-th solution in GI-HSS. However, LGI-HSS can choose the same $i$-th solution by evaluating only a small number of solutions. In Fig. 7, we show the number of examined solutions to choose each solution (i.e., the $i$-th solution for $i$ = 1, 2, ..., 100) in a single run of LGI-HSS on the eight-objective DTLZ2 and I-DTLZ2 problems where the candidate solution size is 10,000. The single run with the median computation time among 11 runs on each problem is selected in Fig. 7. For comparison, the results by GI-HSS are also shown, which is the line specified by 10,000$-(i-1)$. For selecting the first solution, both algorithms examine all the given 10,000 solutions. After that, the proposed LGI-HSS algorithm can significantly decrease the number of solution evaluations. On the eight-objective DTLZ2 problem, the number of solution evaluations decreases with fluctuation as the number of selected solutions increases. On the eight-objective I-DTLZ2 problem, the average number of solution evaluations needed to select each solution is much smaller than that on the eight-objective DTLZ2 problem. This difference is the reason for the difference in the results in Table II for these two problems.

In the same manner as in Fig. 7, we show the results by GI-IGDSS and LGI-IGDSS in Fig. 8, and the results by GI-IGD+SS and LGI-IGD+SS in Fig. 9. We can observe from these two figures that the number of solution evaluations needed to select the first 50 solutions is much larger than that for the remaining 50 solutions. We can also see that the difference in the results on the two problems in these two figures is much smaller than that in Fig. 7. This is the reason why similar results were obtained for all problems by LGI-IGDSS and LGI-IGD+SS. 

\begin{figure}[htbp]
\centering
\includegraphics[width=0.33\textwidth,trim=10 10 10 20,clip]{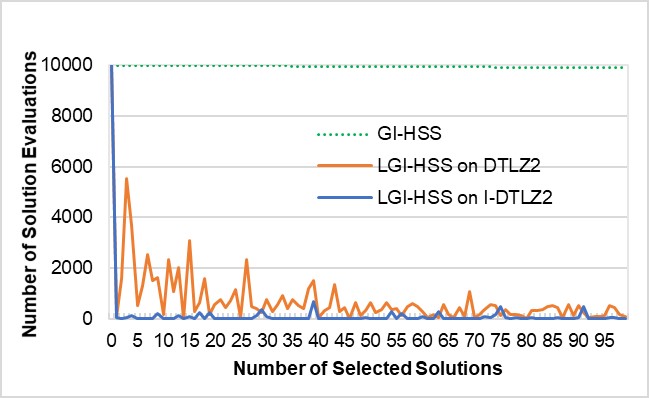}
\caption{The number of solution evaluations to select each solution used by GI-HSS and LGI-HSS.}
\end{figure}
\begin{figure}[htbp]
\centering
\includegraphics[width=0.33\textwidth,trim=10 10 10 20,clip]{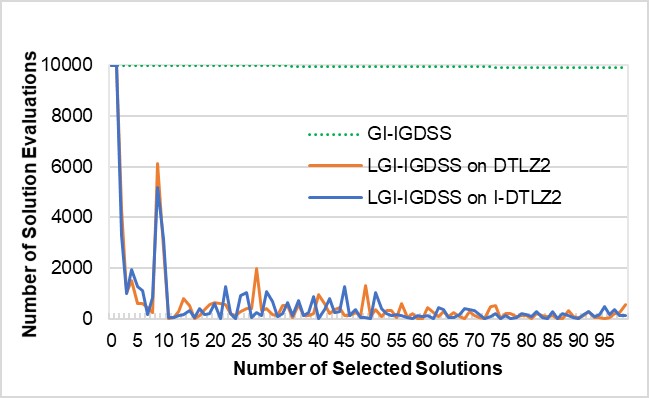}
\caption{The number of solution evaluations to select each solution used by GI-IGDSS and LGI-IGDSS.}
\end{figure}
\begin{figure}[htbp]
\centering
\includegraphics[width=0.33\textwidth,trim=10 10 10 20,clip]{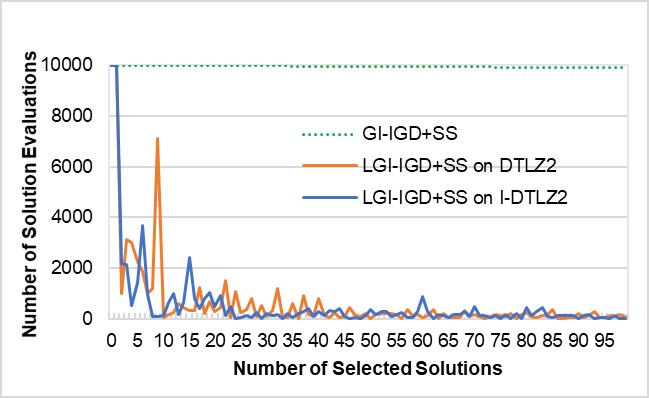}
\caption{The number of solution evaluations to select each solution used by GI-IGD+SS and LGI-IGD+SS.}
\end{figure}

In the worst case, the number of solution evaluations in the proposed algorithms is the same as the standard greedy algorithms (i.e., the worst time complexity of the proposed algorithms is the same as that of the standard greedy algorithms). However, as shown in Figs. 7-9, the number of actually evaluated solutions is much smaller than the worst-case upper bound (which is shown by the dotted green line in each figure). 

In order to examine the effect of the number of objectives on the efficiency of the proposed algorithm, we show experimental results by the proposed three algorithms on the five-objective and ten-objective I-DTLZ2 problems in Figs. 10-12. As in Figs. 7-9, a single run with the median computation time is selected for each algorithm on each test problem. In Fig. 10 by LGI-HSS, much fewer solutions are examined for the ten-objective problem than the five-objective problem. As a result, the average computation time on the ten-objective I-DTLZ2 problem by LGI-HSS in Table I was shorter than that on the five-objective I-DTLZ2 problem. In Fig. 11 by LGI-IGDSS, the difference in the number of solution evaluations between the two problems is not so large if compared with Fig. 10. As a result, the average computation time by LGI-IGDSS increased with the increase in the number of objectives in Table II. The results by LGI-IGD+SS in Fig. 12 are between Fig. 10 and Fig. 11.

\subsection{Size of Candidate Solution Sets}
In the previous subsection, 100 solutions were selected from 5000-20000 solutions of test problems with five to ten objectives. Our algorithms are proposed for choosing a pre-specified number of solutions from an unbounded external archive. In order to show the size of the subset selection problem in this scenario, we monitor the number of non-dominated solutions among all the examined solutions by an EMO algorithm on the four test problems with eight objectives. As an EMO algorithm, we use NSGA-II\cite{NSGAII} and NSGA-III\cite{NSGAIII}. The default parameter settings in \cite{NSGAII} are used in the algorithms (e.g., the population size is specified as 156). The number of non-dominated solutions is counted at the 200th, 400th, 600th, 800th and 1000th generations. Average results over 11 runs are shown in Fig. 13. From these figures, we can see that hundreds of thousands of non-dominated solutions can be obtained by EMO algorithms for many-objective problems. This observation supports the necessity of highly-efficient subset selection algorithms.

\begin{figure}[tbp]
\centering
\includegraphics[width=0.37\textwidth, height=0.2\textwidth,trim=10 10 10 20,clip]{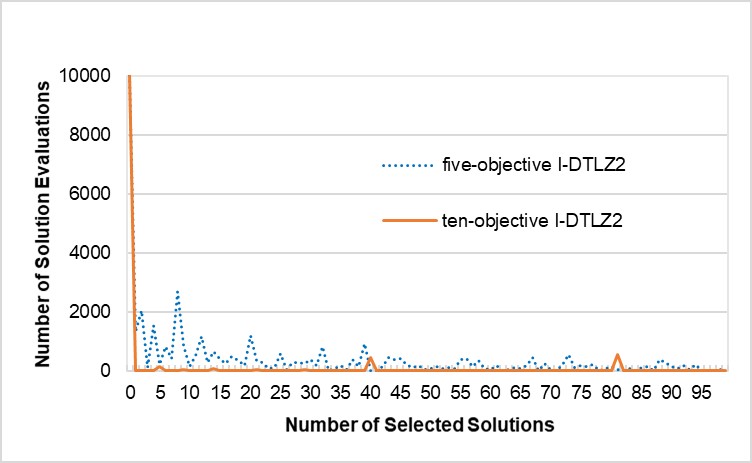}
\caption{The number of solution evaluations to select each solution used by LGI-HSS on five-objective and ten-objective I-DTLZ2 problems.}
\end{figure}
\begin{figure}[htbp]
\centering
\includegraphics[width=0.37\textwidth,trim=10 10 10 20,clip]{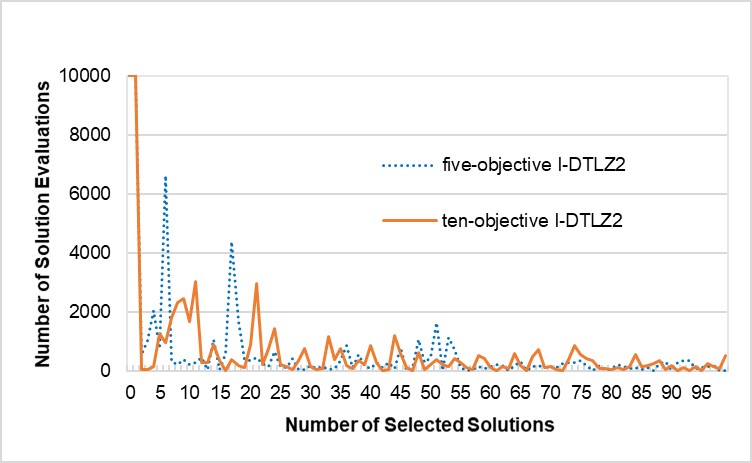}
\caption{The number of solution evaluations to select each solution used by LGI-IGDSS on five-objective and ten-objective I-DTLZ2 problems.}
\end{figure}
\begin{figure}[htbp]
\centering
\includegraphics[width=0.37\textwidth,trim=10 10 10 20,clip]{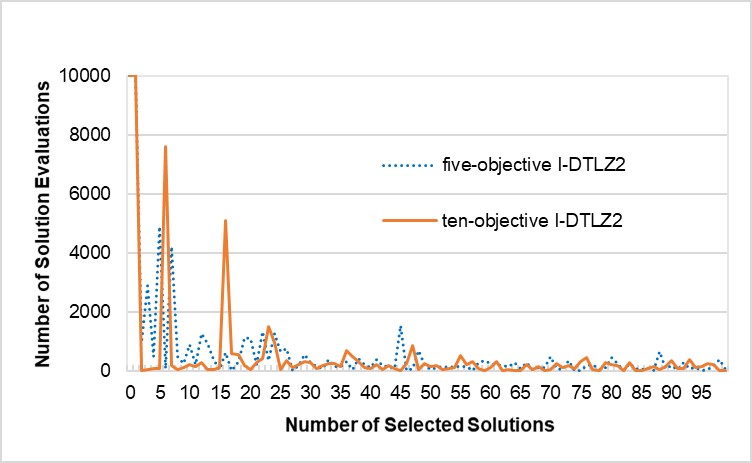}
\caption{The number of solution evaluations to select each solution used by LGI-IGD+SS on five-objective and ten-objective I-DTLZ2 problems.}
\end{figure}
\subsection{Performance on real-world solution sets}
We examine the performance of the three HSS algorithms on four real-world problems (i.e., car side impact design problem \cite{NSGAIII2}, conceptual marine design problem \cite{RE42}, water resource planning problem \cite{water} and car cab design problem \cite{NSGAIII}). We apply NSGA-II \cite{NSGAII} and NSGA-III \cite{NSGAIII} with an unbounded external archive to each problem. The population size is specified as 220 for the four-objective problems, 182 for the six-objective problem, and 210 for the nice-objective problem. After the 200\textit{th} generation of each algorithm, non-dominated solutions in the archive are used as candidate solution sets. In this manner, we generate eight candidate solution sets. Then, each hypervolume-based subset selection algorithm is applied to each candidate solution set to select 100 solutions. This experiment is iterated 11 time. The average size of the candidate solution sets and the average computation time of each subset selection algorithm are shown in Table III.

\begin{figure}[htbp]
\centering
  \subfigure[NSGA-II.]{\includegraphics[width=0.35\textwidth,trim=10 10 10 20,clip]{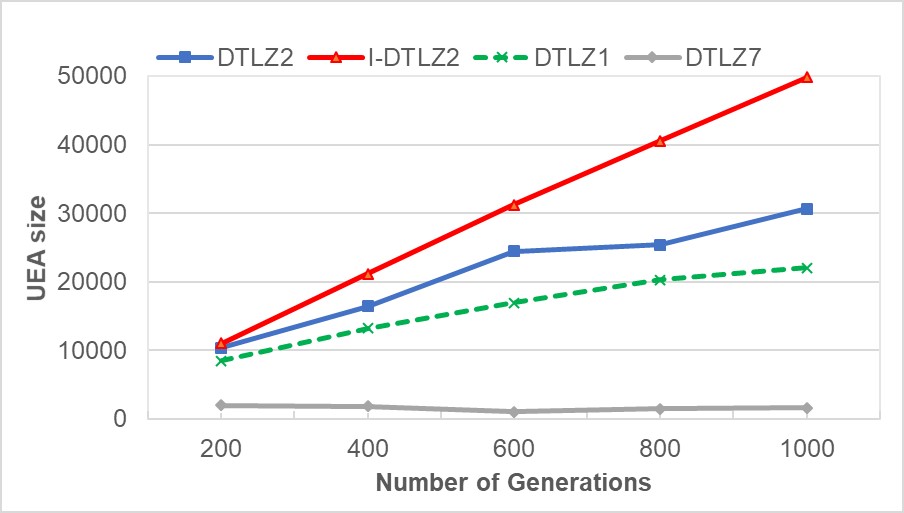}}
  \subfigure[NSGA-III.]{\includegraphics[width=0.35\textwidth,trim=5 10 10 20,clip]{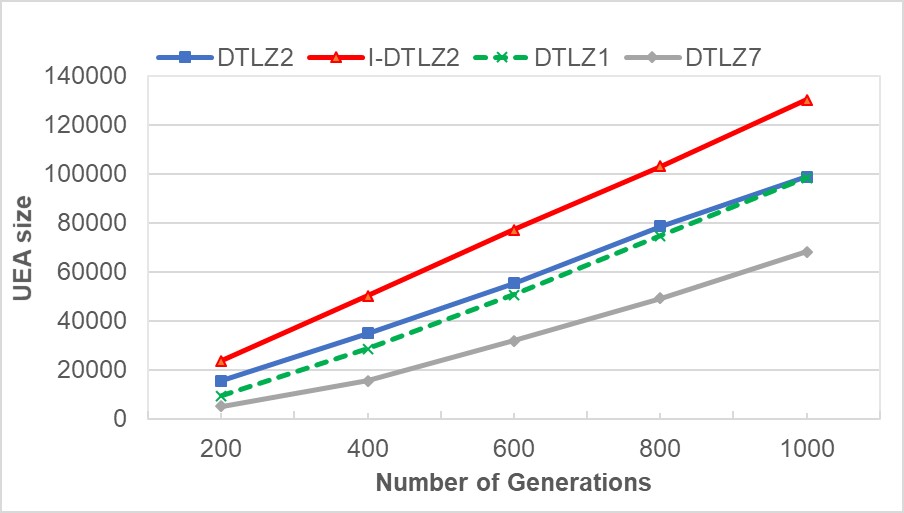}}
  \caption{ The average number of non-dominated
solutions among all the examined solutions at each generation. }
  %\label{fig1}
\end{figure}
We can see from Table III that the proposed algorithms always achieve the best results (i.e., the shortest average computation time) for all the eight settings. Especially when the number of objectives is large (e.g., the nine-objective car cab design problem), the proposed algorithm can significantly decrease the computation time for subset selection. This is consistent with the experimental results on artificially generated solution sets. Similar significant improvement by the proposed algorithms in the computation time is also observed for subset selection based on the IGD and IGD+ indicators whereas they are not included due to the paper length limitation.

\begin{table}[]
\centering
\caption{Average Computation Time (in Seconds) on Real-world Solution Sets Over 11 Runs. The Best Results are Highlighted by Bold.}
\label{tab:my-table}
\begin{tabular}{p{90pt}p{20pt}p{20pt}|p{13pt}p{13pt}p{13pt}}
\toprule
Solution Set                                     & Size  & \# Obj. & GI-HSS & UGI-HSS & LGI-HSS  \\ \hline
NSGA-II on the car side impact design problem    & 10913 & 4             & 84.0   & 10.2    & \textbf{7.0}      \\
NSGA-III on the car side impact design problem   & 11453 & 4             & 92.7   & 13.0    & \textbf{8.3}          \\
NSGA-II on the conceptual marine design problem  & 8549  & 4             & 62.4   & 6.1     & \textbf{4.8}     \\
NSGA-III on the conceptual marine design problem & 10328 & 4             & 81.6   & 8.0     & \textbf{6.5}   \\
NSGA-II on the water resource planning problem   & 12284 & 6             & 117.3  & 22.1    & \textbf{9.8}      \\
NSGA-III on the water resource planning problem  & 19985 & 6             & 238.1  & 38.3    & \textbf{23.2}   \\
NSGA-II on the car cab design problem            & 22196 & 9             & 412.5  & 81.5    & \textbf{15.3}    \\
NSGA-III on the car cab design problem           & 23958 & 9             & 516.3  & 97.6    & \textbf{21.4}   \\ \bottomrule
\end{tabular}
\end{table}

\section{Concluding Remarks}
In this paper, we proposed efficient greedy inclusion algorithms to select a small number of solutions from a large candidate solution set for hypervolume maximization, and IGD and IGD+ minimization. The proposed algorithms are based on the submodular property of the three indicators. The core idea of these algorithms is to use the submodular property to avoid unnecessary contribution calculations. The proposed lazy greedy algorithm finds the same solution subset as the standard greedy inclusion algorithm for each indicator since our algorithm does not change the basic framework of greedy inclusion. Our experimental results on the four test problems (DTLZ1, DTLZ2, DTLZ7 and Inverted DTLZ2) with five, eight and ten objectives showed that the proposed three greedy algorithms are much faster than the standard greedy inclusion algorithms. Besides, the proposed LGI-HSS is much faster than the state-of-the-art fast greedy inclusion algorithm. 

Our experimental results clearly showed that the proposed idea can drastically decrease the computation time (e.g., to 6-9\% of the computation time by the standard greedy inclusion algorithm for IGD+). The proposed idea is applicable to any performance indicator with the submodular property such as hypervolume, IGD and IGD+. However, when we use hypervolume approximation instead of exact calculation, the calculated hypervolume indicator is not strictly submodular. One interesting future research topic is to apply the proposed idea to approximate hypervolume subset selection algorithms in order to examine the quality of obtained subsets by lazy approximate algorithms. Another interesting research direction is to accelerate the proposed algorithms by further explore the special properties of these indicators. Besides, in distance-based greedy subset selection\cite{11Singh}, the distance of each candidate solution to the selected set does not increase by the addition of a new solution to the selected solution set. Thus, using similar lazy idea to reduce the number of distance calculations is also a promising future work. 

The source code of the proposed lazy greedy subset selection algorithms is available from https://github.com/weiyuchen1999/LGISS.

% or


% ============================================
%\appendices
%Appendix one text goes here %\cite{Roberg2010}.

%\section{}

%\section*{Acknowledgment}

%The authors would like to thank D. Root for the loan of the SWAP. The SWAP that can ONLY be usefull in Boulder...

\ifCLASSOPTIONcaptionsoff
  \newpage
\fi



%\begin{thebibliography}{1}


\bibliographystyle{IEEEtran}
\bibliography{tevc_ref}

\begin{thebibliography}{10}
\providecommand{\url}[1]{#1}
\csname url@rmstyle\endcsname
\providecommand{\newblock}{\relax}
\providecommand{\bibinfo}[2]{#2}
\providecommand\BIBentrySTDinterwordspacing{\spaceskip=0pt\relax}
\providecommand\BIBentryALTinterwordstretchfactor{4}
\providecommand\BIBentryALTinterwordspacing{\spaceskip=\fontdimen2\font plus
\BIBentryALTinterwordstretchfactor\fontdimen3\font minus
  \fontdimen4\font\relax}
\providecommand\BIBforeignlanguage[2]{{%
\expandafter\ifx\csname l@#1\endcsname\relax
\typeout{** WARNING: IEEEtran.bst: No hyphenation pattern has been}%
\typeout{** loaded for the language `#1'. Using the pattern for}%
\typeout{** the default language instead.}%
\else
\language=\csname l@#1\endcsname
\fi
#2}}

\bibitem{1multicriteria}
M.~Ehrgott, \emph{Multicriteria Optimization}.\hskip 1em plus 0.5em minus
  0.4em\relax Springer Science \& Business Media, 2005, vol. 491.

\bibitem{2Zitzler}
E.~{Zitzler}, L.~{Thiele}, M.~{Laumanns}, C.~M. {Fonseca}, and V.~G. {da
  Fonseca}, ``Performance assessment of multiobjective optimizers: an analysis
  and review,'' \emph{IEEE Transactions on Evolutionary Computation}, vol.~7,
  no.~2, pp. 117--132, 2003.

\bibitem{3Hisao}
H.~{Ishibuchi}, Y.~{Setoguchi}, H.~{Masuda}, and Y.~{Nojima}, ``How to compare
  many-objective algorithms under different settings of population and archive
  sizes,'' in \emph{2016 IEEE Congress on Evolutionary Computation}, 2016, pp.
  1149--1156.

\bibitem{4Natarajan}
B.~K. Natarajan, ``Sparse approximate solutions to linear systems,'' \emph{SIAM
  Journal on Computing}, vol.~24, no.~2, pp. 227--234, 1995.

\bibitem{5Davis}
G.~Davis, ``Adaptive greedy approximations,'' \emph{Constructive
  Approximation}, vol.~13, no.~1, pp. 57--98, 1997.

\bibitem{6Nemhauser}
G.~L. Nemhauser, L.~A. Wolsey, and M.~L. Fisher, ``An analysis of
  approximations for maximizing submodular set functions-{I},''
  \emph{Mathematical Programming}, vol.~14, no.~1, pp. 265--294, 1978.

\bibitem{7Bradstreet}
L.~{Bradstreet}, L.~{While}, and L.~{Barone}, ``Incrementally maximising
  hypervolume for selection in multi-objective evolutionary algorithms,'' in
  \emph{2007 IEEE Congress on Evolutionary Computation}, 2007, pp. 3203--3210.

\bibitem{8Qian}
C.~Qian, Y.~Yu, and Z.-H. Zhou, ``Subset selection by pareto optimization,'' in
  \emph{Advances in Neural Information Processing Systems 28}, C.~Cortes, N.~D.
  Lawrence, D.~D. Lee, M.~Sugiyama, and R.~Garnett, Eds.\hskip 1em plus 0.5em
  minus 0.4em\relax Curran Associates, Inc., 2015, pp. 1774--1782.

\bibitem{9Hisao}
H.~{Ishibuchi}, H.~{Masuda}, and Y.~{Nojima}, ``Selecting a small number of
  non-dominated solutions to be presented to the decision maker,'' in
  \emph{2014 IEEE International Conference on Systems, Man, and Cybernetics
  (SMC)}, 2014, pp. 3816--3821.

\bibitem{10Tanabe}
R.~{Tanabe}, H.~{Ishibuchi}, and A.~{Oyama}, ``Benchmarking multi- and
  many-objective evolutionary algorithms under two optimization scenarios,''
  \emph{IEEE Access}, vol.~5, pp. 19\,597--19\,619, 2017.

\bibitem{11Singh}
H.~K. {Singh}, K.~S. {Bhattacharjee}, and T.~{Ray}, ``Distance-based subset
  selection for benchmarking in evolutionary multi/many-objective
  optimization,'' \emph{IEEE Transactions on Evolutionary Computation},
  vol.~23, no.~5, pp. 904--912, Oct 2019.

\bibitem{12Chen}
W.~{Chen}, H.~{Ishibuchi}, and K.~{Shang}, ``Modified distance-based subset
  selection for evolutionary multi-objective optimization algorithms,'' in
  \emph{2020 IEEE Congress on Evolutionary Computation.}, 2020.

\bibitem{13hype}
J.~Bader and E.~Zitzler, ``{HypE: A}n algorithm for fast hypervolume-based
  many-objective optimization,'' \emph{Evolutionary Computation}, vol.~19,
  no.~1, pp. 45--76, 2011.

\bibitem{14Bringmann}
K.~Bringmann, T.~Friedrich, and P.~Klitzke, ``Generic postprocessing via subset
  selection for hypervolume and epsilon-indicator,'' in \emph{International
  Conference on Parallel Problem Solving from Nature}.\hskip 1em plus 0.5em
  minus 0.4em\relax Springer, 2014, pp. 518--527.

\bibitem{15Kuhn}
T.~Kuhn, C.~M. Fonseca, L.~Paquete, S.~Ruzika, M.~M. Duarte, and J.~R.
  Figueira, ``Hypervolume subset selection in two dimensions: Formulations and
  algorithms,'' \emph{Evolutionary Computation}, vol.~24, no.~3, pp. 411--425,
  2016.

\bibitem{16Guerreiro}
A.~P. Guerreiro and C.~M. Fonseca, ``Computing and updating hypervolume
  contributions in up to four dimensions,'' \emph{IEEE Transactions on
  Evolutionary Computation}, vol.~22, no.~3, pp. 449--463, 2017.

\bibitem{17Ishibuchi}
H.~Ishibuchi, L.~M. Pang, and K.~Shang, ``Solution subset selection for final
  decision making in evolutionary multi-objective optimization,'' 2020.

\bibitem{18bringmann}
K.~Bringmann, T.~Friedrich, and P.~Klitzke, ``Two-dimensional subset selection
  for hypervolume and epsilon-indicator,'' in \emph{Proceedings of the 2014
  Annual Conference on Genetic and Evolutionary Computation}, 2014, pp.
  589--596.

\bibitem{19bringmann2010approximating}
K.~Bringmann and T.~Friedrich, ``Approximating the volume of unions and
  intersections of high-dimensional geometric objects,'' \emph{Computational
  Geometry}, vol.~43, no. 6-7, pp. 601--610, 2010.

\bibitem{Miqing}
M.~Li and X.~Yao, ``An empirical investigation of the optimality and
  monotonicity properties of multiobjective archiving methods,'' in
  \emph{Evolutionary Multi-Criterion Optimization}, K.~Deb, E.~Goodman, C.~A.
  Coello~Coello, K.~Klamroth, K.~Miettinen, S.~Mostaghim, and P.~Reed,
  Eds.\hskip 1em plus 0.5em minus 0.4em\relax Cham: Springer International
  Publishing, 2019, pp. 15--26.

\bibitem{landscape}
A.~{Liefooghe}, F.~{Daolio}, S.~{Verel}, B.~{Derbel}, H.~{Aguirre}, and
  K.~{Tanaka}, ``Landscape-aware performance prediction for evolutionary
  multi-objective optimization,'' \emph{IEEE Transactions on Evolutionary
  Computation}, pp. 1--1, 2019.

\bibitem{pang2020algorithm}
L.~M. Pang, H.~Ishibuchi, and K.~Shang, ``Algorithm configurations of moea/d
  with an unbounded external archive,'' 2020.

\bibitem{reverse}
Y.~{Nan}, K.~{Shang}, H.~{Ishibuchi}, and L.~{He}, ``Reverse strategy for
  non-dominated archiving,'' \emph{IEEE Access}, vol.~8, pp.
  119\,458--119\,469, 2020.

\bibitem{TANABE2018}
R.~Tanabe and H.~Ishibuchi, ``An analysis of control parameters of moea/d under
  two different optimization scenarios,'' \emph{Applied Soft Computing},
  vol.~70, pp. 22 -- 40, 2018.

\bibitem{NSGAIII}
K.~{Deb} and H.~{Jain}, ``An evolutionary many-objective optimization algorithm
  using reference-point-based nondominated sorting approach, part {I}:
  {Solving} problems with box constraints,'' \emph{IEEE Transactions on
  Evolutionary Computation}, vol.~18, no.~4, pp. 577--601, 2014.

\bibitem{45HSS}
W.~{Chen}, H.~{Ishibuchi}, and K.~{Shang}, ``Lazy greedy hypervolume subset
  selection from large candidate solution sets,'' in \emph{2020 IEEE Congress
  on Evolutionary Computation}, 2020, pp. 1149--1156.

\bibitem{20knowles2003bounded}
J.~D. Knowles, D.~W. Corne, and M.~Fleischer, ``Bounded archiving using the
  lebesgue measure,'' in \emph{2003 IEEE Congress on Evolutionary
  Computation.}, vol.~4.\hskip 1em plus 0.5em minus 0.4em\relax IEEE, 2003, pp.
  2490--2497.

\bibitem{21Zitzler}
E.~Zitzler and L.~Thiele, ``Multiobjective optimization using evolutionary
  algorithms-a comparative case study,'' in \emph{International Conference on
  Parallel Problem Solving from Nature}.\hskip 1em plus 0.5em minus 0.4em\relax
  Springer, 1998, pp. 292--301.

\bibitem{KeHypervolume}
K.~{Shang}, H.~{Ishibuchi}, L.~{He}, and L.~M. {Pang}, ``A survey on the
  hypervolume indicator in evolutionary multi-objective optimization,''
  \emph{IEEE Transactions on Evolutionary Computation}, pp. 1--1, 2020.

\bibitem{22durillo2010jmetal}
J.~J. Durillo, A.~J. Nebro, and E.~Alba, ``The jmetal framework for
  multi-objective optimization: Design and architecture,'' in \emph{2010 IEEE
  Congress on Evolutionary Computation}.\hskip 1em plus 0.5em minus 0.4em\relax
  IEEE, 2010, pp. 1--8.

\bibitem{23durillo2011jmetal}
J.~J. Durillo and A.~J. Nebro, ``jmetal: A java framework for multi-objective
  optimization,'' \emph{Advances in Engineering Software}, vol.~42, no.~10, pp.
  760--771, 2011.

\bibitem{24}
M.~H. {Overmars} and C.~{Yap}, ``New upper bounds in klee's measure problem,''
  in \emph{29th Annual Symposium on Foundations of Computer Science}, Oct 1988,
  pp. 550--556.

\bibitem{25overmars1991new}
M.~H. Overmars and C.-K. Yap, ``New upper bounds in klee's measure problem,''
  \emph{SIAM Journal on Computing}, vol.~20, no.~6, pp. 1034--1045, 1991.

\bibitem{26beume2009s}
N.~Beume, ``S-metric calculation by considering dominated hypervolume as klee's
  measure problem,'' \emph{Evolutionary Computation}, vol.~17, no.~4, pp.
  477--492, 2009.

\bibitem{27while2011fast}
L.~While, L.~Bradstreet, and L.~Barone, ``A fast way of calculating exact
  hypervolumes,'' \emph{IEEE Transactions on Evolutionary Computation},
  vol.~16, no.~1, pp. 86--95, 2011.

\bibitem{29Bradstreet2008fast}
L.~Bradstreet, L.~While, and L.~Barone, ``A fast incremental hypervolume
  algorithm,'' \emph{IEEE Transactions on Evolutionary Computation}, vol.~12,
  no.~6, pp. 714--723, 2008.

\bibitem{28Rote2016selecting}
G.~Rote, K.~Buchin, K.~Bringmann, S.~Cabello, and M.~Emmerich, ``Selecting k
  points that maximize the convex hull volume,'' in \emph{Proceedings of the
  19th Japan Conference on Discrete and Computational Geometry, Graphs, and
  Games}, 2016, pp. 58--60.

\bibitem{37ulrich2012bounding}
T.~Ulrich and L.~Thiele, ``Bounding the effectiveness of hypervolume-based
  ($\mu$+ $\lambda$)-archiving algorithms,'' in \emph{International Conference
  on Learning and Intelligent Optimization}.\hskip 1em plus 0.5em minus
  0.4em\relax Springer, 2012, pp. 235--249.

\bibitem{43bringmann2010efficient}
K.~Bringmann and T.~Friedrich, ``An efficient algorithm for computing
  hypervolume contributions,'' \emph{Evolutionary Computation}, vol.~18, no.~3,
  pp. 383--402, 2010.

\bibitem{30cox2016improving}
W.~Cox and L.~While, ``Improving the {IWFG} algorithm for calculating
  incremental hypervolume,'' in \emph{2016 IEEE Congress on Evolutionary
  Computation}.\hskip 1em plus 0.5em minus 0.4em\relax IEEE, 2016, pp.
  3969--3976.

\bibitem{31jiang2014simple}
S.~Jiang, J.~Zhang, Y.-S. Ong, A.~N. Zhang, and P.~S. Tan, ``A simple and fast
  hypervolume indicator-based multiobjective evolutionary algorithm,''
  \emph{IEEE Transactions on Cybernetics}, vol.~45, no.~10, pp. 2202--2213,
  2014.

\bibitem{32Coello2004A}
C.~A. Coello~Coello and M.~Reyes~Sierra, ``A study of the parallelization of a
  coevolutionary multi-objective evolutionary algorithm,'' in \emph{Mexican
  International Conference on Artificial Intelligence}, 2004.

\bibitem{33Ishibuchi}
H.~Ishibuchi, H.~Masuda, Y.~Tanigaki, and Y.~Nojima, ``Modified distance
  calculation in generational distance and inverted generational distance,'' in
  \emph{Evolutionary Multi-Criterion Optimization}, A.~Gaspar-Cunha,
  C.~Henggeler~Antunes, and C.~C. Coello, Eds.\hskip 1em plus 0.5em minus
  0.4em\relax Cham: Springer International Publishing, 2015, pp. 110--125.

\bibitem{44IGDShape}
H.~{Ishibuchi}, R.~{Imada}, Y.~{Setoguchi}, and Y.~{Nojima}, ``Reference point
  specification in inverted generational distance for triangular linear pareto
  front,'' \emph{IEEE Transactions on Evolutionary Computation}, vol.~22,
  no.~6, pp. 961--975, 2018.

\bibitem{34Lopez}
E.~M. {Lopez} and C.~A.~C. {Coello}, ``{IGD+-EMOA}: {A} multi-objective
  evolutionary algorithm based on {IGD+},'' in \emph{2016 IEEE Congress on
  Evolutionary Computation}, 2016, pp. 999--1006.

\bibitem{35IGD}
Y.~{Sun}, G.~G. {Yen}, and Z.~{Yi}, ``{IGD} indicator-based evolutionary
  algorithm for many-objective optimization problems,'' \emph{IEEE Transactions
  on Evolutionary Computation}, vol.~23, no.~2, pp. 173--187, 2019.

\bibitem{36HowToReference}
H.~Ishibuchi, R.~Imada, Y.~Setoguchi, and Y.~Nojima, ``How to specify a
  reference point in hypervolume calculation for fair performance comparison,''
  \emph{Evolutionary Computation}, vol.~26, no.~3, pp. 411--440, 2018.

\bibitem{38minoux1978accelerated}
M.~Minoux, ``Accelerated greedy algorithms for maximizing submodular set
  functions,'' in \emph{Optimization Techniques}.\hskip 1em plus 0.5em minus
  0.4em\relax Springer, 1978, pp. 234--243.

\bibitem{39leskovec2007cost}
J.~Leskovec, A.~Krause, C.~Guestrin, C.~Faloutsos, J.~VanBriesen, and
  N.~Glance, ``Cost-effective outbreak detection in networks,'' in
  \emph{Proceedings of the 13th ACM SIGKDD International Conference on
  Knowledge Discovery and Data Mining}, 2007, pp. 420--429.

\bibitem{47NetworkInference}
M.~Gomez-Rodriguez, J.~Leskovec, and A.~Krause, ``Inferring networks of
  diffusion and influence,'' \emph{ACM Transactions on Knowledge Discovery from
  Data}, vol.~5, no.~4, Feb. 2012.

\bibitem{41DTLZ}
K.~{Deb}, L.~{Thiele}, M.~{Laumanns}, and E.~{Zitzler}, ``Scalable
  multi-objective optimization test problems,'' in \emph{2002 IEEE Congress on
  Evolutionary Computation.}, vol.~1, May 2002, pp. 825--830 vol.1.

\bibitem{42IDTLZ}
H.~{Jain} and K.~{Deb}, ``An evolutionary many-objective optimization algorithm
  using reference-point based nondominated sorting approach, part{ II}:
  Handling constraints and extending to an adaptive approach,'' \emph{IEEE
  Transactions on Evolutionary Computation}, vol.~18, no.~4, pp. 602--622, Aug
  2014.

\bibitem{46Comparision}
H.~Ishibuchi, R.~Imada, N.~Masuyama, and Y.~Nojima, ``Comparison of
  hypervolume, {IGD} and {IGD+} from the viewpoint of optimal distributions of
  solutions,'' in \emph{Evolutionary Multi-Criterion Optimization}, K.~Deb,
  E.~Goodman, C.~A. Coello~Coello, K.~Klamroth, K.~Miettinen, S.~Mostaghim, and
  P.~Reed, Eds.\hskip 1em plus 0.5em minus 0.4em\relax Cham: Springer
  International Publishing, 2019, pp. 332--345.

\bibitem{NSGAII}
K.~{Deb}, A.~{Pratap}, S.~{Agarwal}, and T.~{Meyarivan}, ``A fast and elitist
  multiobjective genetic algorithm: {NSGA-II},'' \emph{IEEE Transactions on
  Evolutionary Computation}, vol.~6, no.~2, pp. 182--197, 2002.

\bibitem{NSGAIII2}
H.~{Jain} and K.~{Deb}, ``An evolutionary many-objective optimization algorithm
  using reference-point based nondominated sorting approach, part ii: Handling
  constraints and extending to an adaptive approach,'' \emph{IEEE Transactions
  on Evolutionary Computation}, vol.~18, no.~4, pp. 602--622, 2014.

\bibitem{RE42}
M.~G. Parsons and R.~L. Scott, ``Formulation of multicriterion design
  optimization problems for solution with scalar numerical optimization
  methods,'' \emph{Journal of Ship Research}, vol.~48, no.~1, pp. 61--76, 2004.

\bibitem{water}
T.~Ray, K.~Tai, and K.~C. Seow, ``Multiobjective design optimization by an
  evolutionary algorithm,'' \emph{Engineering Optimization}, vol.~33, no.~4,
  pp. 399--424, 2001.

\end{thebibliography}

\vfill


\end{document}